\documentclass{article}

\usepackage[numbers]{natbib}
 \usepackage[preprint]{neurips_2024}

\usepackage{wrapfig}

\usepackage[utf8]{inputenc} 
\usepackage[T1]{fontenc}    
\usepackage{hyperref}       
\usepackage{url}            
\usepackage{booktabs}       
\usepackage{amsfonts}       
\usepackage{nicefrac}       
\usepackage{microtype}      
\usepackage{xcolor}         

\usepackage{amsmath}
\usepackage{dsfont}
\RequirePackage{amsthm}
\RequirePackage{mathtools} 
\RequirePackage{amssymb} 
\RequirePackage{bbm}
\RequirePackage{bm}
\RequirePackage{xfrac}
\RequirePackage[ruled]{algorithm2e}
\usepackage{our_macros}
\usepackage{enumitem}
\usepackage{caption}
\usepackage{subcaption}
\usepackage{footmisc}

\hypersetup{
  colorlinks = true,
  urlcolor = blue,
  linkcolor = blue,
  citecolor = blue
}

\newif\ifneurips
\neuripstrue

\newif\ifaistats
\aistatsfalse

\title{Model-free Low-Rank Reinforcement Learning\\ via Leveraged Entry-wise Matrix Estimation}

\author{%
  Stefan Stojanovic \\
  KTH, Stockholm, Sweden\\
  \texttt{stesto@kth.se} \\
  \And
  Yassir Jedra \\
  MIT, Cambridge, USA \\
  \texttt{jedra@mit.edu} \\
  \AND
  Alexandre Proutiere \\
  KTH, Digital Futures, Stockholm, Sweden\\
  \texttt{alepro@kth.se} \\
}

\begin{document}

\maketitle

\begin{abstract}
We consider the problem of learning an $\varepsilon$-optimal policy in controlled dynamical systems with low-rank latent structure. For this problem, we present \lorapi\ (Low-Rank Policy Iteration), a model-free learning algorithm alternating between policy improvement and policy evaluation steps. In the latter, the algorithm estimates the low-rank matrix corresponding to the (state, action) value function of the current policy using the following two-phase procedure. The entries of the matrix are first sampled uniformly at random to estimate, via a spectral method, the {\it leverage scores} of its rows and columns. These scores are then used to extract a few important rows and columns whose entries are further sampled. The algorithm exploits these new samples to  complete the matrix estimation using a CUR-like method. For this leveraged matrix estimation procedure, we establish entry-wise guarantees that remarkably, do not depend on the coherence of the matrix but only on its spikiness. These guarantees imply that \lorapi\ learns an $\varepsilon$-optimal policy using $\widetilde{O}({S+A\over \mathrm{poly}(1-\gamma)\varepsilon^2})$ samples where $S$ (resp. $A$) denotes the number of states (resp. actions) and $\gamma$ the discount factor. Our algorithm achieves this order-optimal (in $S$, $A$ and $\varepsilon$) sample complexity under milder conditions than those assumed in previously proposed approaches.  
\end{abstract}

\section{Introduction}\label{sec:intro}

Reinforcement Learning (RL) methods when applied to dynamical systems with large state and action spaces suffer from the curse of dimensionality. For example, learning an $\varepsilon$-optimal policy in tabular discounted Markov Decision Processes (MDPs) with $S$ states and $A$ actions requires a number of samples scaling at least as ${SA\over (1-\gamma)^3\varepsilon^2}$ \citep{gheshlaghi2013minimax,sidford2018}. Fortunately, many real-world systems exhibit a latent structure that if learnt and exploited could drastically improve the statistical efficiency of RL methods \citep{laskin20a, stooke21a}. In this paper, we are interested in developing methods to leverage low-rank latent structures. These structures have attracted a lot of attention recently, see e.g. \citep{jiang2017,dann2018,DuKJAD19a,misra2020kinematic,foster2021,zhang2022BMDP, sun2019,agarwal2020flambe,Modi21, uehara2022, ren2022spectral,shah2020sample,stojanovic2024spectral,sam2023overcoming}. Here, we consider a structure where the (state, action) value functions of policies, viewed as $S\times A$ matrices, are low-rank. This structure has been empirically motivated and studied in \citep{shah2020sample, sam2023overcoming,yang2020Harnessing,rozada2024}. The hope is that when exploiting it optimally, learning an $\varepsilon$-optimal policy would only require $O({S+A\over (1-\gamma)^3\varepsilon^2})$ samples. Such an improvement would also imply significant statistical gains in MDPs with continuous state and action spaces. If these spaces are of dimensions $d_1$ and $d_2$, under natural smoothness conditions and using an appropriate discretization \citep{shah2020sample}, the sample complexity would indeed be reduced from ${1\over \varepsilon^{d_1+d_2+2}}$ (without structure) to ${1\over \varepsilon^{\max(d_1,d_2)+2}}$.

In this paper, we present \lorapi\ (Low Rank Policy Iteration), a model-free algorithm that learns and exploits an initially hidden low-rank structure in MDPs. Unlike existing algorithms, \lorapi\ does not require any prior information on the structure. Yet, the algorithm offers the promised statistical gains: its sample complexity essentially exhibits an order-optimal dependence in $S$, $A$ and $\varepsilon$ (i.e., ${S+A\over \varepsilon^2}$).

{\bf Contributions.} Our algorithm \lorapi\ relies on approximate policy iteration  \cite{bertsekas1996neuro}. As such, it alternates between policy evaluation and policy improvement steps. The design and performance analysis of these two steps constitute our main contributions. 

{\it 1. Leveraged matrix estimation with entry-wise guarantees.} \lorapi\ sequentially updates a candidate policy whose (state, action) value function has to be estimated. This function can be seen as an $S\times A$ matrix that we consider to be low rank. The policy evaluation step then boils down to a novel low-rank matrix estimation procedure. We have two main constraints for this procedure. (i) To be sample efficient, the matrix should be estimated from  noisy observations of only a few of its entries. (ii) For RL purposes (when integrated to \lorapi), the procedure should offer entry-wise performance guarantees. We present \lme\ (Leveraged Matrix Estimation), a low-rank matrix estimation algorithm that meets these constraints. \lme\ does not require knowledge of a priori unknown parameters of the matrix (such as its rank, condition number, spikiness, or coherence), and it is the first algorithm enjoying non-vacuous entry-wise guarantees even for coherent matrices. 
    \begin{table}[h!]
    \centering
    \small 
    \setlength{\tabcolsep}{4pt} 
    \resizebox{\columnwidth}{!}{%
    \begin{tabular}{||c c c c c||} 
     \hline
     \textbf{Method} & \textbf{Err. Guarantees} & \textbf{Sampling} &  \textbf{Assumption} & \textbf{Complexity} \\ 
     \hline\hline
     \lme\  (ours) & entry-wise & adaptive  & bounded spikiness & $\alpha^2 (S+A)/\epsilon^2$ \\ 
     Algorithm 1 \cite{shah2020sample} & entry-wise & apriori fixed anchors & anchors apriori known & $\alpha^2 (S+A)/\epsilon^2$ \\ 
     LR-EVI (Thm 9 \cite{sam2023overcoming}) & entry-wise & unif. anchors  & incoherence &  $\mu^2 \alpha^2 (S+A)/\epsilon^2$ \\ 
     NNM \cite{chen2020noisy} (Thm 21 \cite{sam2023overcoming}) & entry-wise & unif. anchors & incoherence & $\mu^2 \alpha^2 (S+A)/\epsilon^2$ \\ 
     Two-phase MC \cite{chen2015completing} & exact recovery & adaptive  & noiseless & not applicable \\ [.5ex] 
     \hline
    \end{tabular}%
    }
    \vspace{0.2cm}
    \caption{Comparison of methods with entry-wise guarantees. For brevity, the factors $(1-\gamma)^{-1}, \kappa$ and $d$ are omitted. NNM: nuclear norm minimization, MC: matrix completion.}
    \label{table_methods}
    \end{table} 

\vspace{-0.75cm}
More precisely, \lme\ guarantees an entry-wise estimation error within $\varepsilon$ using only $\widetilde{O}\left( \kappa^4 \alpha^2 \frac{ (S+A) + \alpha^2}{(1-\gamma)^3 \varepsilon^2} \right)$ samples, where $\alpha$ and $\kappa$ denote the spikiness and the condition number of the matrix, respectively. Note that in particular, this sample complexity does not depend on the coherence of the matrix. Its  dependence in $S$, $A$ and $\varepsilon$ cannot be improved. To reach this level of performance, \lme\ relies on an adaptive sampling strategy. It first estimates, via a spectral method, the so-called {\it leverage scores} of the matrix. These scores quantify the amount of information about the matrix available in the different rows and columns. The algorithm then exploits the leverage scores to adapt its strategy and in turn, drive the sampling process towards more informative entries.

{\it 2. Design and sample complexity of \lorapi.} Our RL algorithm \lorapi\ is a policy iteration algorithm that relies on \lme\ to perform policy evaluation steps. The algorithm inherits the advantages of \lme. In contrast to existing algorithms, it is parameter-free and its performance can be analyzed and guaranteed under mild assumptions on the (state, actions) value functions. In particular, the corresponding low-rank matrices do not need to be incoherent. We establish that \lorapi\ learns an $\varepsilon$-optimal policy using $\widetilde{O}\left( \kappa^4 \alpha^2 \frac{ (S+A) + \alpha^2}{(1-\gamma)^8 \varepsilon^2} \right)$ samples, where $\alpha$ and $\kappa$ are upper bounds on the spikiness and the condition number of the (state, action) value functions.   

{\it 3. Numerical experiments.} We illustrate numerically the performance of our algorithms, \lme\ and \lorapi, using synthetically generated low-rank MDPs. The experiments are presented in Appendix~\ref{app:num} due to space constraints.

\paragraph{Notation.} We denote the Euclidean norm of a vector $x$ by $\Vert x \Vert_2$. Let $M$ be an $m \times n$ matrix. We 
we denote its $i$-th row (resp. $j$-th column) by $M_{i,:}$ (resp. by $M_{:, j}$). We denote its operator norm by  $\Vert M \Vert_{\op}$, it Frobenius norm by $\Vert M \Vert_{\F}$, its infinity norm by $\Vert M \Vert_\infty = \max_{i \in [m],j\in [n]} \vert M_{ij}\vert$, and its two-to-infinity norm by $\Vert M \Vert_{2 \to \infty} = \max_{i \in m} \Vert M_{i,:}\Vert_2 $. We denote by $M^\dagger$ the Moore-Penrose inverse of $M$. For given subsets $\cI \in [m]$, $\cJ \in [n]$, we denote by $M_{\cI, \cJ}$ the sub-matrix whose entries are $\lbrace M_{ij}: (i,j) \in \cI \times \cJ \rbrace$. Finally, we use $a\wedge b = \min (a, b)$ and $a\vee b = \max (a, b)$. 

\section{Related Work}\label{sec:related}

{\bf Low-rank MDPs.} MDPs with low-rank latent structure have been extensively studied recently. We may categorize these studies according to the type of the underlying low-rank structure and to the nature of the algorithms used to learn this structure. 

The most studied low-rank structure concerns MDPs whose transition kernels and the expected reward functions are low-rank. For instance, it is assumed that the transition probabilities can be written as $p(s'|s,a) = \phi(s,a)^\top\mu(s')$, where $\phi(s,a)$ and $\mu(s')$ are $d$-dimensional feature maps \citep{jiang2017,dann2018,DuKJAD19a,misra2020kinematic,foster2021,zhang2022BMDP, sun2019,agarwal2020flambe,Modi21, uehara2022, ren2022spectral}. These work additionally assume that the feature map $\phi$ (and similarly for $\mu$) belongs to a rich function class ${\cal H}$. In this setting, the typical upper bounds derived for the sample complexity of identifying an $\varepsilon$-optimal policy scale as $\textrm{poly}(A, (1-\gamma)^{-1}){\log|{\cal H}|\over \varepsilon^2}$. When no restrictions are imposed on the class ${\cal H}$, one can find a low-rank structure such that $\log|{\cal H}|$ scales as the number $S$ of states \cite{jedra2023nearly}. In this case, the aforementioned upper bounds are the same those for MDPs without structure. We also note that most algorithms using this framework rely on strong computational oracles (e.g., empirical risk minimizers, maximum likelihood estimators), see \cite{kane2022computational, golowich2022learning, zhang2022making} for detailed discussions. In this paper, we do not limit our analysis to low-rank structures based on a given restricted class of functions, and our algorithms do not rely on any kind of oracle. 

The low-rank structure we consider is similar to that in \cite{shah2020sample, sam2023overcoming} and just assumes that the (state, action) value functions are low-rank. Actually, \citep{shah2020sample} considers the case where only the optimal Q-function is low-rank, say of rank $d$. As shown in \cite{shah2020sample}, such a structure naturally arises when discretizing smooth MDPs with continuous state and action spaces. In both papers \cite{shah2020sample, sam2023overcoming}, the authors devise algorithms with a minimax-optimal sample complexity to identify an $\varepsilon$-optimal policy roughly scaling as $(S+A)/\varepsilon^2$. But the analysis presented in \cite{shah2020sample} suffers from the following important limitations. 1. First, it is assumed that the learner is aware of a set $\Ianchors$ (resp. $\Janchors$) of so-called anchors states (resp. actions), such that the rank of the matrix $Q_{\Ianchors,\Janchors}:=(Q(s,a))_{(s,a)\in \Ianchors \times \Janchors}$ is the same as that of the entire matrix $Q$. Such anchors are however initially unknown (since $Q$ is unknown). Importantly, the proposed RL algorithms rely on a low-rank matrix estimation procedure whose performance strongly depends on the smallest singular value $\sigma_d(Q_{\Ianchors,\Janchors})$ of $Q_{\Ianchors,\Janchors}$. The authors circumvent this difficulty by actually parametrizing their algorithms using $\sigma_d(Q_{\Ianchors,\Janchors})$. But again, the latter is unknown, and it remains unclear how one can avoid this issue. 2. The second  limitation is that the analysis is valid for small values of the discount factor $\gamma$ (the authors need to impose an upper bound on $\gamma/\sigma_d(Q_{\Ianchors,\Janchors})$). When $\sigma_d(Q_{\Ianchors,\Janchors})$ is small, the analysis is limited to very short horizons. Note that in addition, \citep{shah2020sample} assumes that the collected rewards are deterministic, which together with the short horizon issue, greatly simplifies the learning problem. 

To address the first limitation, the authors of \citep{sam2023overcoming} propose to sample rows and columns uniformly at random to get anchors. This solution requires to sample at least $d\mu^2$ states and actions (Lemma 10 in \cite{sam2023overcoming}) where $\mu$ is the (unknown) coherence of the matrix to be estimated. Hence this essentially amounts to sampling almost the whole matrix for coherent matrices. The authors of \citep{sam2023overcoming} also propose a solution to the second limitation, but at the expense of imposing additional restrictive conditions. In this paper, we address both limitations and devise RL algorithms that rely on a new low-rank matrix estimation procedure that works without imposing the incoherence of the matrix and that does not require knowledge on a priori unknown parameters of this matrix. 

{\bf Low-rank matrix estimation with entry-wise guarantees.} Until recently, most results on low-rank matrix recovery concerned guarantees with respect to the spectral or Frobenius norms, see e.g. \citep{davenport2016} and references therein. Over the past few years, methods to derive entry-wise guarantees have been developed. These include spectral approaches \cite{abbe2020entrywise, chen2021spectral, stojanovic2024spectral}, nuclear-norm penalization and convex optimization techniques \cite{chen2020noisy}, CUR-based (or Nyström-like) methods \cite{shah2020sample, agarwal2023causal, sam2023overcoming}. 

The aforementioned literature provides guarantees not for all low-rank matrices, but for those typically enjoying additional structural properties such as incoherence. Relaxing the incoherence assumption is not easy, but can be achieved using adaptive sampling \cite{krishnamurthy2013low,chen2015completing,wang2018provably}. As far as we are aware, all results applicable to somewhat coherent matrices provide guarantees with respect to the spectral or Frobenius norms. In this paper, we develop a first adaptive matrix estimation method with provable entry-wise guarantees, valid for matrices with well-defined spikiness but not necessarily incoherent. Refer to \citep{mackey2011distributed, negahban2012restricted} and to \textsection\ref{subsec:coherence} for a detailed discussion about the notion of spikiness.

\section{Preliminaries}\label{sec:model}

\subsection{Low-rank Markov Decision Processes} 

We consider a discounted MDP with finite state and action spaces ${\cal S}$ and ${\cal A}$. These spaces are of cardinality $S$ and $A$, respectively. The dynamics are described by the transition kernel $p$ where $p(s'|s,a)$ denotes the probability to move to state $s'$ given current state $s$ and that the action $a$ is selected. The collected rewards are random but bounded by $r_{\max}$, and $r(s,a)$ is the expected reward collected when action $a$ is selected in state $s$. A deterministic Markovian policy $\pi$ is described by a mapping from ${\cal S}$ to ${\cal A}$. We denote by $V^\pi$ the state value function of $\pi$: for all $s\in {\cal S}$, $V^\pi(s) = \mathbb{E}[\sum_{t=0}^\infty\gamma^t r(s_t^\pi,a_t^\pi) | s_0^\pi=s]$, where $s_t^\pi$ and $a_t^\pi$ are, at time $t$, the state and the action selected under $\pi$. Similarly, the (state, action) value function of $\pi$ is defined by: for all $(s,a)\in {\cal S}\times {\cal A}$, $Q^\pi(s,a)=r(s,a)+\gamma \sum_{s'}p(s'|s,a)V^\pi(s')$. $Q^\pi$ can be seen as a $S\times A$ matrix, referred to as the {\it value matrix of $\pi$} in the remainder of the paper. Let $\kappa_\pi$ denote the condition number of $Q^\pi$. Finally, let $V^\star$ be the value function of the MDP (the value function of the optimal policy). 

The objective is to learn an $\varepsilon$-optimal policy by interacting with the MDP. Such a policy satisfies: for all $s\in {\cal S}$, $V^\pi(s) \ge V^\star(s)-\varepsilon$. Without any assumption on the structure of the MDP, to identify such a policy, the learner needs to gather, even with a generative model, a number of samples\footnote{Here a sample refers to an experience $(s,a,r,s')$, the observation of the collected reward $r$ and the next state $s'$, starting with a given (state, action) pair $(s,a)$. Under a generative model, the learner can adapt the choice of $(s,a)$ for the next observed experience without any constraint.} that scales as ${SA\over \varepsilon^2(1-\gamma)^{3}}$ \cite{gheshlaghi2013minimax, sidford2018}. The hope is that exploiting an a-priori known structure in the MDP may considerably accelerate the learning process. In this paper, we focus on a low-rank latent structure. Formally, we define:

\begin{definition}[Rank of a policy, rank of the MDP] The rank $d_\pi$ of a deterministic policy $\pi$ is the rank of its value matrix $Q^\pi$. The rank of an MDP is then defined as $d=\max_{\pi}d_\pi$, where the maximum is over all deterministic policies.  
\end{definition}

Throughout the paper, we assume that the MDP is low-rank: its rank $d$ satisfies  $d\ll (S+A)$.  This assumption is merely made to simplify the exposition of our results and proof techniques. As we shall argue in Appendix \ref{subsec:approximately_low_rank}, our findings can naturally be extended to MDPs that are only low-rank in an approximate and well-precised sense.

\subsection{Matrix estimation: coherence and spikiness}\label{subsec:coherence}

Our learning algorithm relies on the approximate policy iteration method, and in particular, in each iteration, it needs to estimate the low-rank value matrix of the current policy. To be sample efficient, the algorithm will estimate the matrix from the noisy observations of a few of its entries. Recovering a low-rank matrix from a few of its entries is not always possible (see e.g. \cite{davenport2016} for a survey), and conditions on the degree to which information about a single entry is spread out across a matrix must be imposed. Examples of such conditions pertain to the  {\it coherence} \cite{candes2010noise, recht2011simpler} or the {\it spikiness} \cite{negahban2012restricted} of the matrix. 

\textbf{Matrix coherence.} Let $Q$ be a rank-$d$ $S\times A$ matrix with SVD $U\Sigma W^\top$. The coherence of $Q$ is defined as $\mu(Q) = \max\{\sqrt{S/d}\|U\|_{2\to\infty},\sqrt{A/d}\|W\|_{2\to\infty}\}$. $Q$ is $\mu$-coherent if $\mu(Q)\le \mu$. 
Low coherence means that the energy of $U$ and $W$ are not concentrated around a few rows and columns.

\textbf{Matrix spikiness.} The spikiness of $Q$ is defined as $\alpha (Q) = \sqrt{SA}\Vert Q \Vert_\infty/\Vert Q \Vert_{\F} \in [1,\sqrt{SA}]$. $Q$ is $\alpha$-spiky if $\alpha (Q)\le \alpha$. 
A matrix has low spikiness if the amplitude of its maximal entry is not much larger than the average amplitude of its entries, in which case, it is intuitively easier to estimate. 

Most existing guarantees for low-rank matrix estimation are expressed through the spectral or Frobenius norm of the error matrix. For this type of guarantees, the estimation error scales polynomially either with the matrix coherence or with its spikiness  \cite{davenport2016, negahban2012restricted}. 
The matrix spikiness was introduced in the matrix completion literature \cite{negahban2012restricted} to obtain guarantees under less restrictive conditions than the incoherence conditions imposed in previous work. Indeed, there are matrices with bounded spikiness but high coherence (say close to $\sqrt{S/d}$, in which case the aforementioned coherence-based guarantees are vacuous). In contrast, bounded incoherence provides an upper bound on spikiness since $
\alpha(Q) = \sqrt{SA} \Vert Q\Vert_{\infty}/\Vert Q\Vert_\F \leq \sqrt{SA} \Vert U\Vert_{2\to\infty} \Vert Q\Vert_{\op} \Vert W \Vert_{2\to\infty} /\Vert Q\Vert_\F \leq
\mu(Q)^2 d$.  

For RL purposes, we need to derive entry-wise guarantees for the estimate of the value matrix of some policy as demonstrated in \cite{shah2020sample, stojanovic2024spectral, jedra2024low}. Existing upper bounds for the entry-wise estimation error exhibit a  strong dependence in the matrix coherence and its condition number, see e.g. \cite{chen2020noisy, chen2021spectral, stojanovic2024spectral}. For instance, in \cite{chen2020noisy}, this dependence comes as a multiplicative factor $\mu(Q)^2 \alpha(Q)^2 \kappa(Q)^2$ in the number of samples required for a given level of estimation accuracy. As far as we are aware, our matrix estimation method is the first able to yield entry-wise guarantees that do not exhibit a dependence on the matrix coherence but only on its spikiness (see Table \ref{table_methods}). Our algorithm is better by a factor of $\mu(Q)^2$ than algorithms based on uniform sampling (studied in \cite{sam2023overcoming}), and requires the same sample complexity as the algorithm of \cite{shah2020sample}, which has prior knowledge of anchor states. It remains however unclear whether the dependence of the entry-wise estimation error in the condition number can be avoided. This last observation guides the design of RL algorithms for low-rank MDPs as we discuss next.

\subsection{Policy vs. Value Iteration: the condition number issue}
\label{subsec:PI_vs_VI}

We aim at devising an algorithm learning an efficient policy with provable guarantees while imposing conditions on the MDP that are as mild as possible. To this aim, one may think of applying either a policy iteration approach, as we do, or a value iteration approach. 

{\it Policy Iteration.} Using this approach, in each iteration, we need to estimate the low-rank value matrix of the current candidate policy. As mentioned above, the entry-wise error of this estimation procedure depends on the condition number of the matrix. Note that this matrix belongs to the finite set of (state, action) value functions of  deterministic policies. As shown in \cite{dadashi2019,wu2022}, this set can be seen as the vertices of a simple polytope ${\cal P}$. Hence to get performance guarantees when applying a PI approach, it is sufficient to impose an upper bound on the condition numbers $\kappa_\pi$ for all deterministic policies $\pi$, or equivalently, on the condition numbers of matrices corresponding to the vertices of ${\cal P}$.  

{\it Value Iteration.} Here, we would maintain, in iteration $t$, an estimate $V^{(t)}$ of the value function $V^\star$, and samples would be used to compute $V^{(t+1)}$, an estimate of ${\cal T}^\star(V^{(t)})$, where ${\cal T}^\star$ denotes Bellman's operator. More precisely,  starting from $V^{(t)}$, we would estimate the low-rank matrix $Q^{(t+1)}={\cal F}(V^{(t)})$ defined by for all $(s,a)$, ${\cal F}(V^{(t)})(s,a)=r(s,a)+\gamma\sum_{s'}p(s'|s,a)V^{(t)}(s')$. Then we would define $V^{(t+1)}$ as the value function of the greedy policy with respect to $Q^{(t+1)}$. Hence to get provable performance guarantees using a value iteration approach, we would need to impose an upper bound on the condition number of $Q^{(t)}$ in all iterations $t$. The main issue is that the set of matrices $\{Q^{(t)}, t\ge 1\}$ is stochastic and hard to predict. Indeed, we have no way of confining the iterates $Q^{(t+1)}$ to the polytope ${\cal P}$: as shown in \cite{dadashi2019}, the polytope is not stable by Bellman's operator. As a consequence, if we wish to get performance guarantees for a value iteration approach, we would need to impose an upper bound on the condition number of all possible matrices of the form ${\cal F}(V)$ for some vector $V$.  

In summary, policy iteration approaches offer a theoretical advantage compared to value iteration. It requires the control of the condition numbers of matrices in a set much smaller than that for value iteration. This advantage is illustrated in Figure \ref{fig1} on a toy example of an MDP. Refer to Appendix \ref{app:num} for additional numerical experiments (with larger MDPs).

\begin{figure}[htb]
  \begin{center}
    \includegraphics[width=0.5\textwidth]{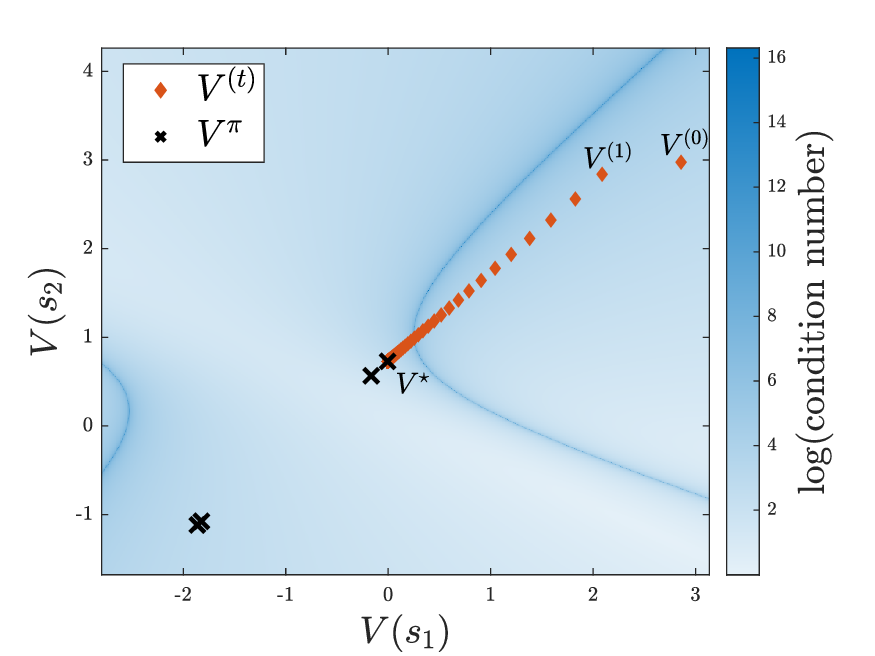}
  \end{center}
  \caption{Consider an MDP with two states and two actions (see Appendix \ref{subsec:app_toy_experiment} for details). The 4 black crosses correspond to the value function of the 4 possible policies. When combining policy iteration with a low rank estimation procedure, we just need to control the condition number of the 4 corresponding value matrices. The red dots correspond to the successive estimates $V^{(t)}$ of $V^\star$ when running value iteration. When applying a value iteration approach, we would need to upper bound the condition number of all the corresponding matrices $Q^{(t)}={\cal F}(V^{(t-1)})$ for $t\ge 1$. For a given $V$, the background color in the figure indicates the value of the condition number of ${\cal F}(V)$. We see that the dynamics of $V^{(t)}$ under the value iteration algorithm are such that the trajectory $(Q^{(t)}, t\ge 1)$ has to go through regions where the condition number is very high. Hence on this example, a value iteration approach would not work well.}\label{fig1}
\end{figure}

\section{Leveraged Matrix Estimation} \label{sec:lme}

In this section, we present Leveraged Matrix Estimation (\lme), an algorithm that estimates the value matrix $Q^\pi$ of a policy $\pi$. The algorithm relies on an active strategy for sampling the entries of the matrix based on its estimated leverage scores as defined below. This active strategy accelerates the learning process and allows us to obtain entry-wise guarantees that do not depend on the coherence of the matrix but on its spikiness only. 

\begin{definition}[Leverage scores\footnote{Our definition of leverage scores is consistent up to a scale factor with that used in the literature \cite{boutsidis2009, drineas2012,chen2015completing}.}]
\label{def:ls} 
Let $Q$ be a rank-$d$ $S\times A$ matrix with SVD $U\Sigma W^\top$. Its left and right leverage scores $\ell$ and $\rho$ are defined as 
$\ell_{s}  = \Vert U_{s,:} \Vert_2^2 / d$ for all $s\in \cS$, and $\rho_{a}  = \Vert W_{a,:} \Vert_2^2/ d$ for all $a \in \cA$.
\end{definition}

\lme\ only takes as inputs a policy $\pi$ and a sampling budget $T$. It proceeds in two phases: first, it uses half of the sampling budget to estimate the leverage scores of $Q^\pi$ via singular subspace recovery. Second, it selects a few anchor rows and columns sampled using the estimated leverage scores, and uses the remaining  budget to sample the entries of these rows and columns. It finally completes the matrix estimation using a CUR-based method. The full pseudo-code of \lme\ is presented in Appendix \ref{app:proof_CUR_theorem}. Observe that \lme\ is parameter-free: it does not require knowledge of the policy rank $d_\pi$, nor upper bounds on unknown parameters such as $\kappa_\pi$ or $\alpha(Q^\pi)$ or $\mu(Q^\pi)$. Throughout this section, when presenting our guarantees, we will abuse notation and use $d$, $\kappa$ and $\alpha$, instead of $d_\pi$, $\kappa_\pi$ and $\alpha(Q^\pi)$.

\subsection{Preliminaries}\label{subsec:prelim}

\lme\ exploits a natural empirical estimator of $Q^\pi$ entries at numerous stages. This empirical estimator is essentially based on Monte-Carlo rollouts with truncation as described next. Define the {truncated value matrix} at a horizon $\tau$ as follows: for all $(s,a) \in \cS \times \cA$,
   $Q^\pi_\tau(s,a) = \EE\left[ \sum_{t=0}^\infty \gamma^t r_t(s_t^\pi , a_t^\pi) \indicator_{\lbrace t \le \tau\rbrace}\big \vert s_0^\pi = s, a_0^\pi = a\right]$.
By choosing $\tau$ appropriately, we may control the level of the approximation error $Q^\pi_\tau - Q^\pi$. 
We make this observation precise in the following lemma, proved in Appendix \ref{subsec:app_truncated_value_lemma}. 
\begin{lemma}\label{lem:truncated-Q}
For any $\epsilon > 0$ and any $\tau \ge  \frac{1}{1- \gamma} \log\left( \frac{r_{\max}}{(1- \gamma) \epsilon}\right)$, we have $\Vert Q^\pi - Q^\pi_\tau  \Vert_{\infty}  \le \epsilon$. 
\end{lemma}

In view of the above, to estimate an entry, say $(s,a)$, of $Q^\pi$, we will use an empirical estimator based on trajectories of length $\tau + 1$ of the system under $\pi$ and starting with (state, action) pair $(s,a)$. In our algorithms, this length is chosen to get an appropriate accuracy level. Specifically, we choose $\epsilon$ and $\tau$ as follows:
\begin{align}\label{eq:eps_tau}
    \epsilon = \frac{r_{\max}}{T} \quad \text{and} \quad \tau = \left\lceil \frac{1}{1-\gamma}\log\left( \frac{ T }{1-\gamma} \right) \right\rceil.
\end{align}
These choices will become apparent from our analysis.

\subsection{Phase 1: Leverage scores estimation via spectral subspace recovery}\label{subsec:l-scores-est}

The first phase of \lme\ is devoted to the estimation of the leverage scores of $Q^\pi$. To this aim, using half of the sampling budget $T/2$, we estimate the singular subspaces of the matrix via a spectral method. 

{\it \underline{Phase 1a.} Data collection and the empirical truncated value matrix.} As suggested in \textsection\ref{subsec:prelim}, to estimate individual entries of $Q^\pi$, we sample system trajectories of length $\tau + 1$. More precisely, for each of the $N:=T/(2(\tau +1))$ trajectories, we first sample the starting (state, action) pair uniformly at random, and then observe the trajectory obtained under the policy $\pi$ and initiated at this pair. The data collected this way is $\cD = \lbrace (s_{k,0}^\pi, a_{k,0}^\pi, r_{k,0}^\pi, \dots, s_{k,\tau}^\pi, a_{k,\tau}^\pi, r_{k,\tau}^\pi): k \in [N] \rbrace$. Using this data, we construct an empirical estimate of the truncated value matrix as follows $\forall (s,a)\in \cS \times \cA$:
\begin{align}\label{eq:emp-truncated-q}
     \widetilde{Q}^\pi_\tau(s,a) & = \frac{SA}{N} \sum_{k=1}^{N}\left(\sum_{t=0}^\tau \gamma^t r_{k, t}^\pi \right)\indicator{\lbrace (s_{k,0}^\pi, a_{k,0}^\pi) = (s,a)\rbrace},
\end{align}

{\it \underline{Phase 1b.} Singular subspace recovery.} We compute the SVD of the empirical truncated value matrix $\widetilde{Q}^\pi_\tau$. We obtain $\widetilde{Q}^\pi_\tau = \sum_{i = 1}^{S\wedge A} \hat{\sigma}_i \hat{u}_i \hat{w}^\top_i$, where $\hat{\sigma}_1, \dots, \hat{\sigma}_{S\wedge A}$ correspond, in decreasing order, to its singular values and $\hat{u}_1, \dots, \hat{u}_{S}$ (resp. $\hat{w}_1, \dots, \hat{w}_{A}$) to its left (resp. right) singular vectors. Using this decomposition, we construct our estimate of $Q^\pi$ as follows: 
\begin{align}\label{eq:est-svt}
    \widehat{Q}^\pi = \sum_{i=1}^{S\wedge A} \hat{\sigma}_i \indicator{\lbrace \hat{\sigma}_i \ge \beta \rbrace}  \hat{u}_i \hat{w}^\top_i,
\end{align}
where $\beta > 0$ is a threshold that we will precise shortly. We view $\widehat{Q}^\pi$ as a biased estimate of $Q^\pi$ with controlled bias through $\tau$. We also use $\beta$ to estimate the rank of $Q^\pi$: $\widehat{d} = \sum_{i=1}^{S\wedge A}  \indicator{\lbrace \hat{\sigma}_i \ge \beta \rbrace}$. Finally, the estimated left (resp. right) singular subspace is denoted $\widehat{U} = \begin{bmatrix}
    \hat{u}_1 & \cdots & \hat{u}_{\hat{d}}
\end{bmatrix} \in \RR^{S \times \hat{d}}$  (resp. $\widehat{W} = \begin{bmatrix}
    \hat{w}_1 & \cdots & \hat{w}_{\hat{d}}
\end{bmatrix} \in \RR^{A \times \hat{d}}$). In the following proposition, we provide a choice for the threshold $\beta$ that yields appropriate guarantees regarding our subspace recovery.

\begin{proposition}
     Let $\delta \in (0,1)$ and choose the threshold $\beta$ as
     \begin{align}
    \beta = \sqrt{\frac{ r_{\max}^2 \cardS \cardA (S+ A)}{ (1-\gamma)^3 T} \log^4\left( \frac{  (S+A)T}{ (1-\gamma)
    \delta}\right) }    + \frac{\Rmax\sqrt{SA}}{T}.
    \label{eq:threshold}
    \end{align}
     Then, provided that\footnote{ To simplify the notation, all our sample complexity guarantees are expressed using $\widetilde{\Omega}_\delta(\cdot)$, the tilde-notation may hide poly-log dependencies in $\delta$, $S$, $A$, $(1-\gamma)^{-1}$, $d$, $\kappa$, $\alpha$, $\log(e/\varepsilon)$, and $r_{\max}$. \label{fn:omega}
}:
      \begin{align}
          T = \widetilde{\Omega}_\delta \left( \frac{ r_{\max}^2  \cardS\cardA } { \sigma_{d}^2(Q^\pi)} \frac{(\cardS+\cardA)}{(1-\gamma)^3} \right)
            \label{eq:T_requirement_beta}
      \end{align}
       we have that events: $\widehat{d} = d_\pi$, and for all $s \in \cS$,
     \begin{align*}
         \Vert U_{s,:} - \hU_{s,:} (\hU^\top U) \Vert_2 \lesssim \frac{r_{\max}\sqrt{\cardS\cardA}}{ (1-\gamma)^{3/2} \sigma_{d}(Q^\pi)}  \left( \sqrt{\frac{d}{T}} + \kappa \Vert U_{s,:} \Vert_2 \sqrt{\frac{\cardS+\cardA}{T}} \right)  \log^2\left(  \frac{  (S+A)T}{(1-\gamma)
        \delta} \right) 
    \end{align*}
    hold with probability at least $1-\delta$.
    An analogous result holds for $\widehat{W}$.
    \label{thm:approx_sing_vectors_guarantee_informal}
\end{proposition}

The precise statement (Theorem \ref{thm:row-wise-guarantee}) and the proof are presented in Appendix \ref{proof:row-wise-guarantee} and \ref{app:subsec_rank_estimation}.

{\it \underline{Phase 1c.} Leverage Scores Estimation.} To conclude, using the recovered subspaces $\widehat{U}$ and $\widehat{W}$, we estimate the leverage scores as follows $\hat{\ell} = \big\Vert \tilde\ell\big\Vert_1^{-1} \tilde{\ell}$ and $\hat{\rho} = \big\Vert \tilde{\rho} \big\Vert_1^{-1}  \tilde{\rho}$, where: 
\ifaistats
\begin{align}\label{eq:est-lev}
        &\forall s\in \states:\ \tilde{\ell}_s = \Vert \widehat{U}_{s,:}\Vert^2 \vee \frac{d}{S}, \\ &\forall a\in \actions:\ \tilde{\rho}_a = \Vert \widehat{W}_{a,:}\Vert^2 \vee \frac{d}{A}.
\end{align}
\fi
\ifneurips
\begin{align}\label{eq:est-lev}
       \forall s\in \states:\ \tilde{\ell}_s = \Vert \widehat{U}_{s,:}\Vert_2^2 \vee \frac{d}{S},\qquad \mathrm{and}\qquad \forall a\in \actions:\ \tilde{\rho}_a = \Vert \widehat{W}_{a,:}\Vert_2^2 \vee \frac{d}{A}.
\end{align}
\fi
The performance of the estimation of the leverage scores is summarized in the following theorem, proved in Appendix \ref{app:proof_theorem_leverage}.  

\begin{theorem}[Leverage Scores Estimation]
    Let $\delta \in (0,1)$. Suppose the threshold $\beta$ is chosen as in \eqref{eq:threshold}. Then, we have that:  $\PP( \forall s \in \cS, \ \;\ell_{s} \le 4 \, \hat{\ell}_s ) \ge 1 -\delta,$ provided that 
    \begin{align*}
        T  = \widetilde{\Omega}_\delta\left( \kappa^2 \frac{r_{\max}^2 SA}{\sigma_{d}^2(Q^\pi) } \frac{(S + A)}{(1- \gamma)^3 }    \right),
    \end{align*}
    An analogous result holds for $\hat{\rho}$.
    \label{thm:leverage_scores}
\end{theorem}

\subsection{Phase 2: Leveraged CUR-based Matrix Completion}
\label{subsec:leveraged_CUR}

Before we proceed with the description of the second phase, we briefly recall the so-called {CUR
decomposition} \cite{goreinov1997theory, mahoney2009cur} for low-rank matrices. The decomposition says that for a given rank-$d$ $S \times A$ matrix $Q$, there always exists $\cI \subseteq [S]$, $\cJ \subseteq [A]$, with $ \vert \cI  \vert = \vert \cJ \vert = d$, such that the sub-matrix $Q_{\cI, \cJ}$ is full rank and for all entries $(i,j)$, $Q_{ij} = Q_{i, \cJ} (Q_{\cI, \cJ})^\dagger Q_{\cI, j}$. As in \cite{shah2020sample, sam2023overcoming, agarwal2023causal}, we leverage this decomposition in our matrix estimation procedure, but without any requirement such as knowledge of $\cI, \cJ$ for which $\sigma_d(Q_{\cI, \cJ})$ bounded away from zero or upper bounds on parameters like the matrix coherence.

{\it \underline{Phase 2a.} Data collection to estimate the skeleton of the value matrix.} We start by sampling $K := 64d \log(64d/\delta)$ rows (resp. columns) without replacement according to $\hat\ell$ (resp. $\hat\rho$) to form a {\it skeleton} of the matrix. These rows and columns are referred to as anchors. We denote the set of selected rows (resp. columns) by $\cI \subseteq \cS$ (resp.  $\cJ \subseteq \cA$). We use the remaining sample budget $T/2$ to get samples of the entries of $Q^\pi$ in the skeleton. To this aim, we use the procedure described in \textsection\ref{subsec:prelim}, and sample trajectories of length $\tau +1$. For each entry $(s,a)\in \Omega_{\square}:= \cI \times \cJ$, we use $N_1 := T/(4(\tau + 1)K^2)$ trajectories to compute $\widetilde{Q}^\pi_\tau(s,a)$, an empirical estimate of ${Q}^\pi(s,a)$ (see (\ref{eq:emp-truncated-q})). For each entry $(s,a)\in \Omega_{+} := ((\cS \backslash \cI)\times \cJ) \cup (\cI \times (\cA \backslash \cJ))$, we use $N_2 := T/(4(\tau + 1)(K (S+A) -2 K^2)$ trajectories. Note that $N_2\le N_1$ (this plays a role in the analysis).

{\it \underline{Phase 2b.} CUR-based completion with Inverse Leverage Scores Weighting.} First, using the leverage scores, and the set of rows $\cI$ and columns $\cJ$,  we define $K \times K$ diagonal matrices $L$ and $R$ as follows:   
\ifaistats
\begin{align}\label{eq:LR-weighting}
    &\forall i \in \cI, \quad L_{ii} = \frac{1}{\min\left\{1,\sqrt{K\hat{\ell}_{i}}\right\} },  \\ &\forall j \in \cJ, \quad R_{jj} = \frac{1}{\min\{1,\sqrt{K\hat{\rho}_{j}} \} }.
\end{align}
\fi
\ifneurips
\begin{align}\label{eq:LR-weighting}
   \forall i \in \cI, \quad L_{ii} = \frac{1}{\min\left\{1,\sqrt{K\hat{\ell}_{i}}\right\} }, \quad \text{and} \quad \forall j \in \cJ, \quad R_{jj} = \frac{1}{\min\{1,\sqrt{K\hat{\rho}_{j}} \} }.
\end{align}
\fi
Next, starting from the values of $\widetilde{Q}^\pi_\tau(s,a)$ for $(s,a)$ in the skeleton, we perform a CUR matrix completion to obtain $\widehat{Q}^\pi$:
\emph{(i)} for all $(s,a) \in (\cS \times \cJ)  \cup (\cI \times \cA)$, we set $\widehat{Q}^\pi(s,a) = \widetilde{Q}^\pi_\tau(s,a)$; \emph{(ii)} for all $(s,a) \in (\cS\backslash \cI) \times (\cA \backslash \cJ)$, we set 
\begin{align}\label{eq:CUR-ME}
    \widehat{Q}^\pi (s,a) & = 
        \widetilde{Q}^\pi_\tau (s, \cJ) R \left( L \, \widetilde{Q}^\pi_\tau(\cI, \cJ) R \right)^\dagger L \, \widetilde{Q}^\pi_\tau(\cI, a).
\end{align}
Note that the use of $L$ and $R$ in \eqref{eq:CUR-ME}, referred to as Inverse Leverage Scores Weighting, corresponds to an importance sampling  procedure. It allows us to account for the fact that the skeleton has been sampled using the (estimated) leverage scores.

The next theorem summarizes the performance guarantees under $\lme$. Its proof is presented in Appendix \ref{subsec:proof-lme-gurantee}.

\begin{theorem}\label{thm:lme-guarantee} Let  $\varepsilon >0$, $\delta \in (0,1)$. Given a deterministic policy $\pi$, and a sampling budget $T$, the algorithm $\lme$ ensures that  $
        \PP(\Vert \widehat{Q}^\pi - Q^\pi \Vert_\infty \le \varepsilon )\ge 1 - \delta
$,
    provided that $\varepsilon \lesssim \Vert Q^\pi \Vert_{\infty}$ and
    \begin{align*}
        T = \widetilde{\Omega}_\delta\left( \frac{ (S+A) + \alpha^2  d   }{ (1-\gamma)^3 \varepsilon^2} (r_{\max}^2  \kappa^4 \alpha^2 d^2  )  \right).
    \end{align*} 
\end{theorem}

Theorem $\ref{thm:lme-guarantee}$ states that the sample complexity of \lme\ to obtain entry-wise guarantees does not depend on the coherence $\mu$ of $Q^\pi$ but rather on its spikiness $\alpha$ and condition number $\kappa$ only. Hence \lme\ provides entry-wise guarantees even for coherent matrices. In addition, its sample complexity scales with $S$, $A$, $\gamma$ and $\varepsilon$ optimally. Indeed if $\alpha, \kappa = {\Theta}(1)$ and $d\ll S+A$, it scales as ${(S+A)\over\varepsilon^2(1-\gamma)^{3}}$. We also wish to emphasize that $\lme$ is parameter-free, in the sense that it does not require knowledge of the so-called anchor rows and columns, nor does it require upper bounds on unknown parameters such as coherence, spikiness, rank or condition number. These properties are desirable for RL purposes. 

\section{Low-Rank Policy Iteration}

In this section, we present and evaluate \lorapi\ (Low Rank Policy Iteration), a model-free variant of the approximate policy iteration algorithm \cite{bertsekas1996neuro}. It alternates between policy improvement and policy evaluation steps and uses \lme, our low rank matrix estimation procedure for policy evaluation. Refer to Algorithm \ref{algo:lora} for the pseudo-code.       

\begin{algorithm}[ht]
    \SetAlgoLined
    \KwIn{sampling budget $T$, accuracy $\varepsilon$, confidence level $\delta$, discount factor $\gamma$, rewards upper bound $r_{\max}$, $\pi^{(1)}$ an initial deterministic policy} 
    Set $N_{\textup{epochs}} \gets (1-\gamma)^{-1} \log\left((4 r_{\max})/((1-\gamma)\varepsilon)\right)$ \\ 
    Set $T_{\textup{eval}} \gets T / N_{\textup{epochs}}$ \\
    \For{$t=1,2,\dots, N_{\mathrm{epochs}}$}{ 
    \emph{(Approximate policy  evaluation).} Obtain the value matrix of $\pi^{(t)}$:
    $
    \widehat{Q}^{(t)} \gets \text{\tt LME}(\pi^{(t)}, T_{\textup{eval}})
    $\\
    \emph{(Policy improvement)} Improve $\pi^{(t)}$ using $\widehat{Q}^{(t)}$:
    $
    \forall s\in \cS, \pi^{(t+1)}(s) \gets \argmax_{a \in \cA} \widehat{Q}^{(t)}(s, a)
    $
    }
    \KwOut{Output $\hat{\pi}=\pi^{(N_{\mathrm{epochs}} + 1)}$.}
    \caption{Low-Rank Policy Iteration ($\lorapi$)}
    \label{algo:lora}
\end{algorithm}

The following theorem provides performance guarantees for \lorapi. 

We state the results under the assumption that for any deterministic policy $\pi$, $Q^\pi$ is $\alpha$-spiky and has a condition number upper bounded by $\kappa$.

\begin{theorem}\label{thm:lora-pi-lme} Let $\delta \in (0,1)$ and $ \varepsilon = \widetilde{O}(\Vert Q^{\pi^{(1)}} \Vert_\infty)$. Under $\lorapi$, we have $\PP\left( \Vert V^\star - V^{\hat{\pi}}  \Vert_\infty \le \varepsilon \right) \ge 1 - \delta $, provided
\begin{align*}
        T = \widetilde{\Omega}_\delta\left( \frac{ (S+A) + \alpha^2  d }{(1-\gamma)^8 \varepsilon^2} (r_{\max}^2  \kappa^4 \alpha^2  d^2)
        \right).
    \end{align*} 
\end{theorem}

The proof of Theorem \ref{thm:lora-pi-lme} is presented in Appendix \ref{app:lora}. Having entry-wise guarantees (as stated in Theorem \ref{thm:lme-guarantee}) at each iteration of the algorithm is critical to establish Theorem \ref{thm:lora-pi-lme}. In fact, the proof starts from the observation (see Lemma \ref{lem:API-convergence}) that 
\ifaistats
\begin{align*}
    (1-\gamma)\Vert V^\star - V^{\hat{\pi}}  \Vert_\infty \le  &2 r_ {\max}\gamma^{N_{\textup{epochs}}} \\ &+ 2 \max_{t \in [N_{\textup{epochs}}]} \Vert \widehat{Q}^{(t)} - Q^{\pi^{(t)}}\Vert_\infty.
\end{align*}
\fi
\ifneurips
\begin{align*}
    (1-\gamma)\Vert V^\star - V^{\hat{\pi}}  \Vert_\infty \le 2 r_ {\max}\gamma^{N_{\textup{epochs}}} + 2 \max_{t \in [N_{\textup{epochs}}]} \Vert \widehat{Q}^{(t)} - Q^{\pi^{(t)}}\Vert_\infty.
\end{align*}
\fi
\lorapi\ combines numerous advantages. (i) It is parameter-free: it does not require the knowledge of upper bounds on parameters such as the ranks, condition numbers, and spikiness of the value matrices of policies. This is thanks to $\lme$, which is itself parameter-free. (ii) Its sample complexity does not depend on the coherence of the value matrices but only on their spikiness; which is an important improvement over existing algorithms \citep{sam2023overcoming}. (iii) \lorapi\ offers performance guarantees without having access to good anchor states and actions, without assuming that the rewards are deterministic and that the discount factor is (far too) small, as in \citep{shah2020sample} (refer to Section \ref{sec:related} for a detailed discussion). (iv) Its sample complexity has an order-optimal scaling in $S$, $A$ and $\varepsilon$. (v) Finally, since \lorapi\ uses policy iteration, its theoretical guarantees can be established under milder assumptions than if value iteration was used instead (see \textsection\ref{subsec:PI_vs_VI}).

The dependence of order $(1-\gamma)^{-8}$ is far from the ideal minimal dependence of order $(1-\gamma)^{-3}$ that one would typically obtain in RL without low-rank structure. This is an artifact of using a model-free approach, and more specifically the Monte-Carlo estimator of entries of the value matrices. Avoiding such high dependence requires further assumptions and a model-based approach.

Furthermore, it is worth mentioning that the guarantees enjoyed by \lorapi\ can be naturally extended to MDPs that are low-rank only in an approximate sense. We refer the reader to Appendix \ref{subsec:approximately_low_rank} for further details.

\section{Conclusion}

In this work, we considered a class of MDPs where the Q-function, viewed as a state-action matrix, admits a low-rank representation under any deterministic policy. We devised $\lorapi$, a model-free learning algorithm based on approximate policy iteration, that provably exploits such low-rank representation to output a near-optimal policy. Critical to the design and performance guarantee of $\lorapi$ is a novel low-rank matrix estimation procedure referred to as $\lme$. $\lme$ is shown to enjoy a tight entry-wise guarantee while being parameter-free, i.e., it does not require knowledge of the so-called anchor rows and columns, nor upper bounds on unknown parameters such as spikiness, coherence, rank, or condition number. More importantly, its sample complexity does not scale with the coherence but instead with the spikiness of the matrix. This allows us to estimate a wider class of low-rank matrices with entry-wise guarantees than previous work. Such desirable properties are what make $\lme$ appealing for RL purposes, and in particular what allows us to show that $\lorapi$ is sample-efficient under mild conditions. From a design perspective, $\lme$ and its analysis features many interesting tools and ideas. Notably, (i) we derived instance-dependent row-wise singular subspace recovery guarantees, and (ii) we combined the use of the so-called leverage scores with a CUR-based approximation for matrix estimation. We believe such tools and ideas to be of independent interest. Finally, we provided experimental results that suggest the superior performance of our proposed algorithms.

\newpage 

\section*{Acknowledgment}

This research was supported by the Wallenberg AI, Autonomous Systems and Software Program
(WASP) funded by the Knut and Alice Wallenberg Foundation, the Swedish Research Council (VR), and Digital Futures. YJ is supported by the Knut and Alice Wallenberg Foundation Postdoctoral Scholarship Program under grant KAW 2022.0366.  

\medskip

\bibliographystyle{plain}
\bibliography{references,references2}


\appendix

\section{Numerical Experiments}\label{app:num} 
All experiments in this section were performed on HP EliteBook 830 G8 with an Intel i7 core and 16 GB of RAM. Each experiment's runtime for individual realizations took at most 2-3 hours, and reproducing all results is feasible within a day.
\subsection{Parameters of the toy example in Figure \ref{fig1}}
\label{subsec:app_toy_experiment}
We considered an MDP with $\cardS = \cardA = 2$, $\gamma = 0.87$, a reward matrix given by
\begin{align*}
    r = \begin{bmatrix} -0.46 & -0.48\\ -0.14 &0.28 \end{bmatrix},
\end{align*}
and the following transition probabilities:
\begin{align*}
    P(s'\vert s,a = a_1) = \begin{bmatrix}
        0.4 &  0.6 \\ 0.15 & 0.85
    \end{bmatrix}
    \qquad 
    P(s'\vert s,a=a_2) = \begin{bmatrix}
        0.25 &  0.75 \\ 0.29 & 0.71
    \end{bmatrix}
\end{align*}
We initialized VI with $V^{(0)} = \begin{bmatrix}2.86 & 2.98 \end{bmatrix}^\top$. Note that $V_{\max} = \frac{\Rmax}{1-\gamma} = 3.69$ and thus $V^{(0)} \in [-V_{\max},V_{\max}]^2$. 

For this example, the condition numbers of the Q-functions induced by policies are $16.08, 4.38, 15.29, 12.07$, while the maximum condition number during value iteration is $\approx 2497.82$.

We stress here that this MDP is full-rank, and the purpose of this example is to demonstrate the potential instability of VI in the presence of large condition numbers. For low-rank MDPs, this corresponds to the matrix $Q^\pi$ having an effectively smaller rank than expected, and estimating all $d$ singular vectors despite $\sigma_1(Q^\pi)/\sigma_d(Q^\pi) \to \infty$.

\subsection{Matrix completion with leveraged anchors}
We consider matrix completion with a fixed matrix $M^\star$ to be estimated, testing four different methods. First, we test a method based on CUR-approximation with anchors chosen uniformly at random. Next, we have a method based on the estimation of leverage scores, where, for a given budget of samples, we use half of them for estimating leverage scores as described in the main text. Then, we consider a method with oracle anchors, where the anchors are chosen with respect to the true leverage scores. Lastly, we consider standard SVD decomposition, where we keep only the first $d$ largest singular values of the matrix.

\begin{figure}[htb]
\centering
\begin{subfigure}{.5\textwidth}
  \centering
  \includegraphics[width=1\linewidth]{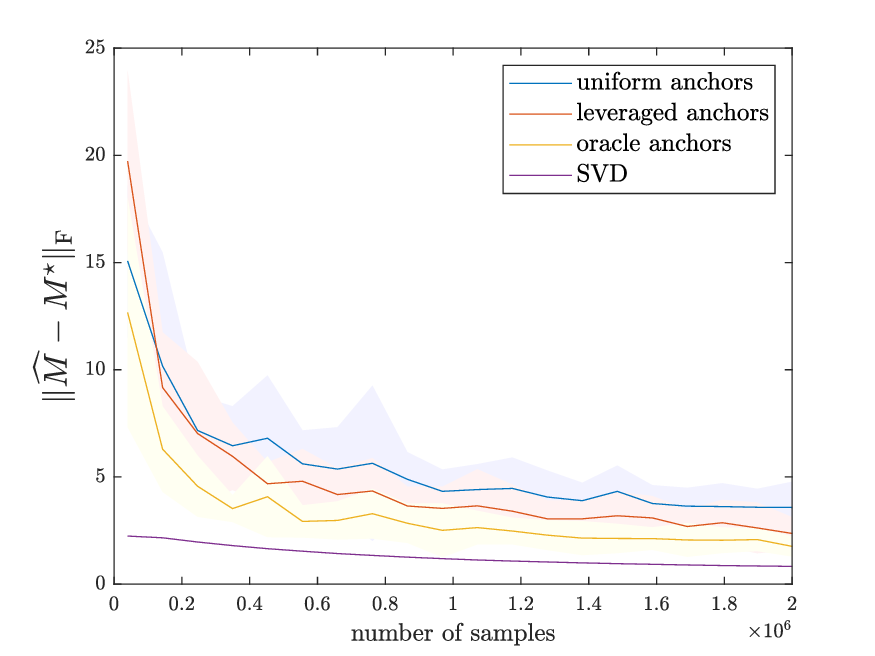}
\end{subfigure}%
\begin{subfigure}{.5\textwidth}
  \centering
  \includegraphics[width=1\linewidth]{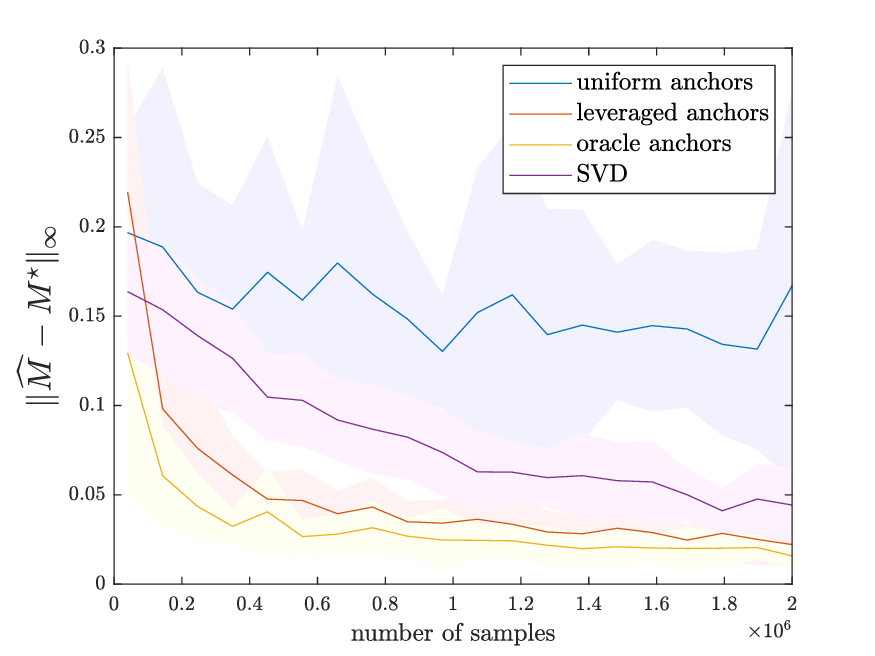}
\end{subfigure}
  \caption{Matrix completion: matrix $M^\star$ is of size $1000\times 1000$, rank $d=5$ and sampled entries have additive Gaussian noise with $\sigma = 0.01$. Number of anchors used was $\nachors = 10$. All plots are averaged over $30$ simulations and a new random matrix $M^\star$ was generated in every $5$ simulations.}
\label{fig_MC}
\end{figure}

As expected, CUR-based methods depend heavily on the quality of anchor selection. The gap between leverage-score-based anchors and oracle anchors is slight, even when half the samples are used to estimate the leverage scores. While SVD shows a smaller Frobenius error, it has higher entrywise error compared to CUR-based methods with good anchors.

\subsection{Leverage scores for VI and PI}
We demonstrate the importance of choosing anchors based on leverage scores for value iteration (VI) and policy iteration (PI). We postpone learning of the anchor states to the next subsections and assume that the true leverage scores of matrices $(Q^{(t)})_{t\geq 1}$ are given. For methods with leveraged anchors, anchors are chosen as those with the highest leverage scores (true leverage scores of $Q^{(t)}$). For uniform anchors, anchors are chosen uniformly without repetitions.  

\begin{figure}[htb]
\centering
\begin{subfigure}{.5\textwidth}
  \centering
  \includegraphics[width=1\linewidth]{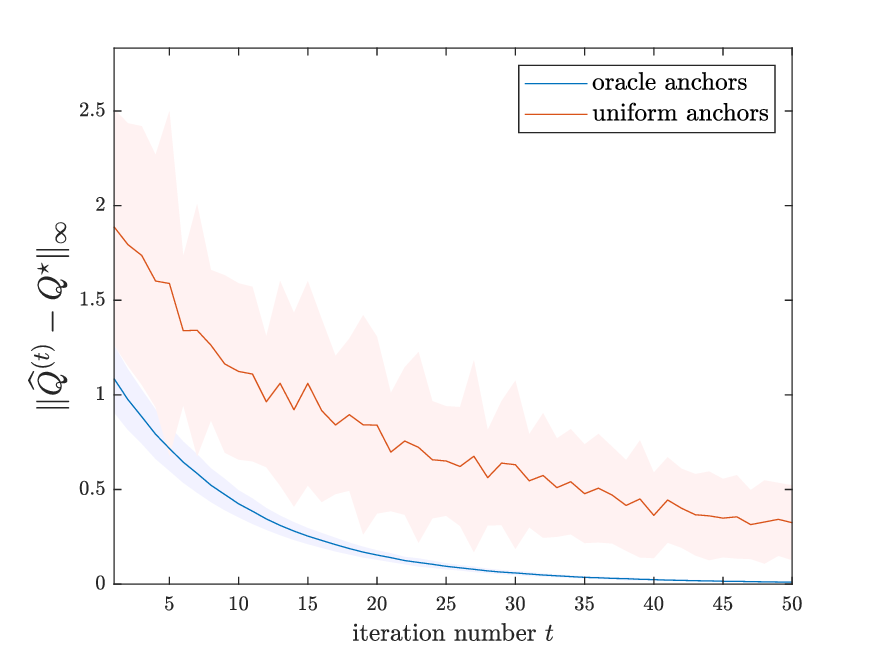}
  \label{fig_LS_VI}
  \caption{Value iteration with anchors.}
\end{subfigure}%
\begin{subfigure}{.5\textwidth}
  \centering
  \includegraphics[width=1\linewidth]{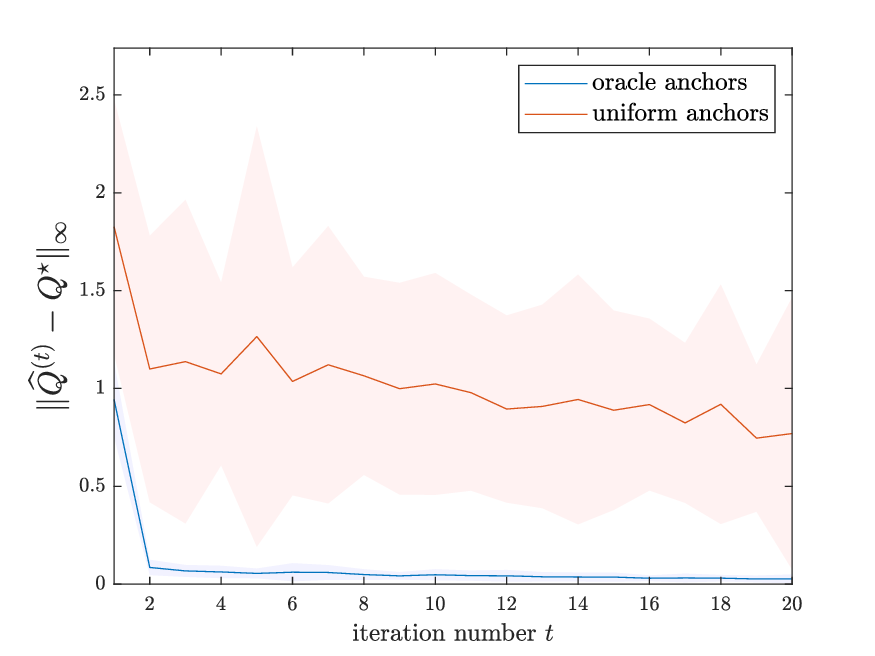}
  \label{fig_LS_PI}
\caption{Policy iteration with anchors.}
\end{subfigure}
  \caption{Matrix $Q^\star$ is obtained from rank $d=5$ rewards and transition matrices. Moreover, $\cardS=70,\cardA=50$, $\gamma = 0.9$, and we choose number of anchors $K = 15$. Observations are noisy with additive Gaussian noise with $\sigma = 0.01$. Plots are averaged over $100$ simulations, and new MDPs are generated every $5$ simulations, while the number of samples in an iteration $t$ is $10(1.1)^t$. }
\label{fig_LS}
\end{figure}

These results highlight that leveraged anchors reduce entrywise error significantly for general matrices. In contrast, uniform anchors show significant randomness, although the error decreases in expectation over iterations.

\subsection{Low-rank Value Iteration}
We evaluate a VI-based variant of Algorithm \ref{algo:lora}, that we refer to as \loravi. We do not assume prior knowledge of the matrices, and use samples to estimate leverage scores and matrices $(Q^{(t)})_{t\geq 1}$. 

\begin{figure}[htb]
\centering
\begin{subfigure}{.5\textwidth}
  \centering
  \includegraphics[width=1\linewidth]{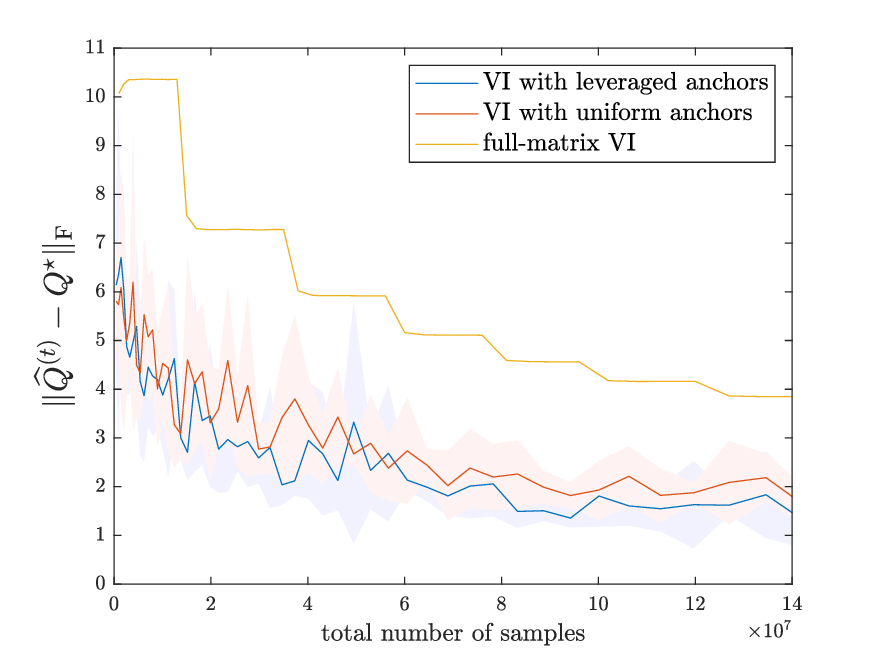}
\end{subfigure}%
\begin{subfigure}{.5\textwidth}
  \centering
  \includegraphics[width=1\linewidth]{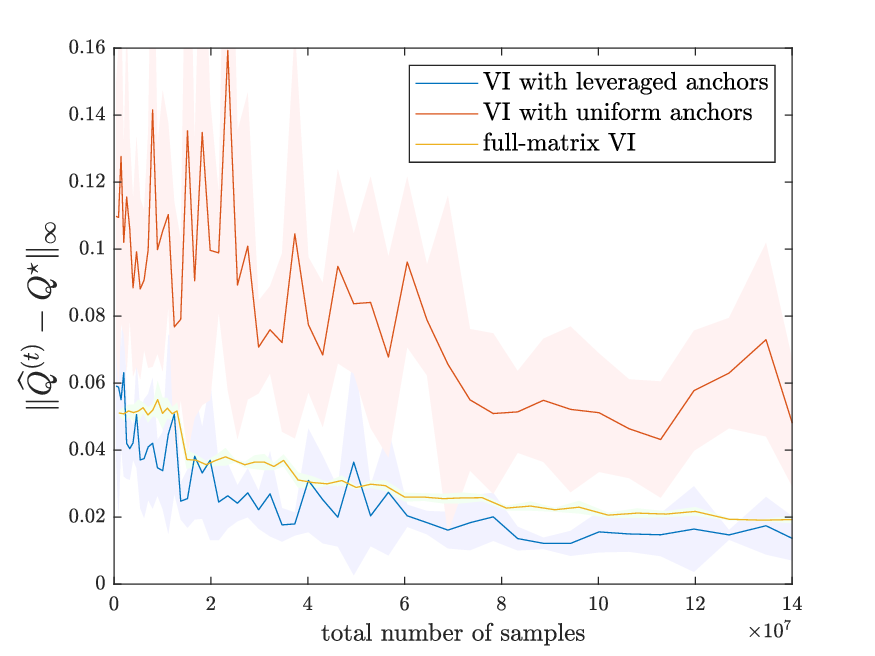}
\end{subfigure}
\caption{\loravi: $Q^\star$ generated from low-rank $r$ and $P$ of rank $d=4$, $S=A=1000$, $\gamma = 0.1$. We used $\nachors = 10$ anchors, $V^{(0)}=0$, rewards are noisy with Gaussian noise $\sigma = 0.01$. All plots are averaged over $5$ simulations, each consisting of $50$ epochs, and the number of samples in an epoch $t$ is approximately $ 20 (1.05)^t (S+A)\nachors$.}
\label{fig_full_VI}
\end{figure}

Even though we did not theoretically analyze the VI-based method in this work, for the reasons mentioned in Section \ref{subsec:PI_vs_VI}, we note that this method works well in practice for the settings considered in this study. We consider three methods: VI with leveraged anchors, where we use half of the experiences to estimate leverage scores and based on them sample the second half in a CUR-like fashion. Next, we consider VI with uniform anchors, where anchors are chosen uniformly at random without repetitions. And finally, we consider full-matrix VI, a standard VI approach without any matrix completion steps, where each entry of the matrix gets observed a certain number of times.

We see in Figure \ref{fig_full_VI} that VI with leveraged anchors achieves the best performance measured in Frobenius and entrywise norm. On the other hand, VI with uniform-anchors does not recover specific entries with high values well (as seen in the right figure), but because there are not too many entries with high values, it achieves decent performance in the Frobenius norm. Finally, even though full-matrix VI can observe all entries of the matrix, it still lags behind VI with leveraged anchors. We also want to remind the reader that VI with leveraged anchors uses only half of the available samples for matrix recovery, while the other half is used for learning the leverage scores.  

The algorithm used in the experimental section of [35] closely resembles our \loravi\ algorithm when uniform anchors are applied. As a result, the numerical results from [35] can be reproduced within our framework, which offers a more general and flexible setting.

\subsection{Low-rank Policy Iteration}
Finally, we experimentally study performance of the proposed algorithm \lorapi.

\begin{figure}[htb]
\centering
\begin{subfigure}{.5\textwidth}
  \centering
  \includegraphics[width=1\linewidth]{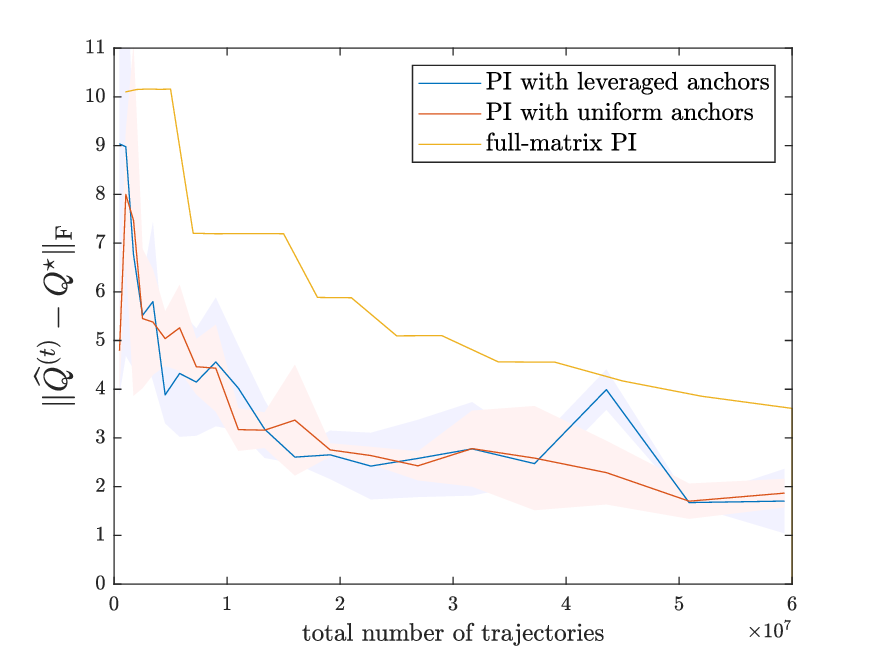}
\end{subfigure}%
\begin{subfigure}{.5\textwidth}
  \centering
  \includegraphics[width=1\linewidth]{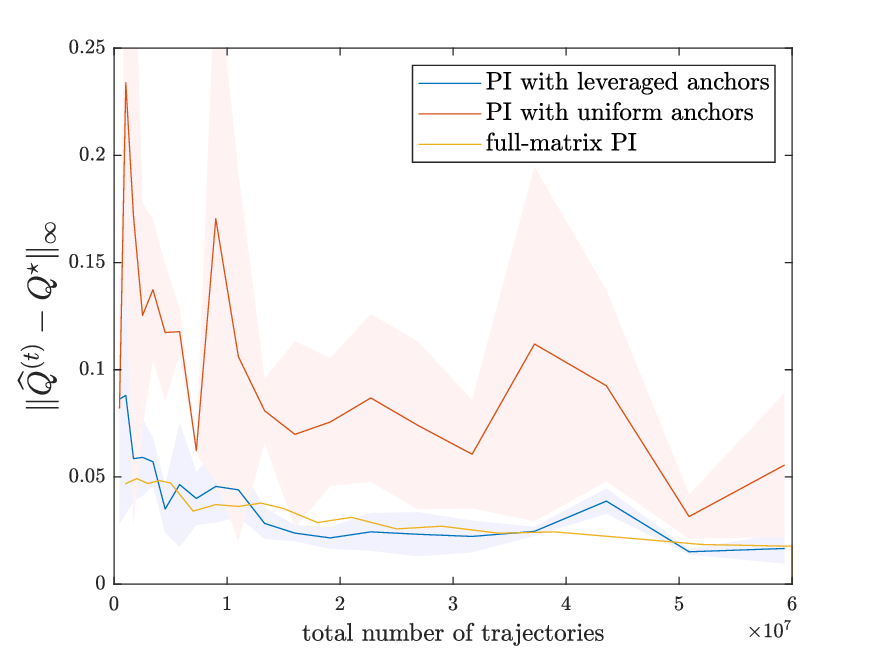}
\end{subfigure}
\caption{\lorapi: $Q^\star$ generated from low-rank $r$ and $P$ of rank $d=4$, $S=A=1000$, $\gamma = 0.1$, $\tau = 5$. We used $\nachors = 10$ anchors, uniformly random initial policy, and noisy rewards with Gaussian noise $\sigma = 0.01$. Plots for PI with anchors are averaged over $3$ simulations, while the one for full-matrix PI is simulated once. Each simulation consisted of $20$ epochs, and the number of samples in an epoch $t$ is approximately $ 10 (1.15)^t (S+A)\nachors$.}
\label{fig_full_PI}
\end{figure}

Similarly as in the previous subsection we study performance of three different methods using PI instead of VI this time, and the observed performance is similar to the one of VI-based methods. In contrast to the other methods, using leverage scores seems to ensure that Frobenius error behaves similarly to entrywise error, up to a scaling factor. This might be caused by uniform dispersion of the estimation error over the entries with large values for PI/VI with leveraged anchors.

The choice of $\gamma = 0.1$ is governed by an observation that this value of parameter $\gamma$ ensures that the largest singular values of $Q^\star$ are scaling similarly. In other words, it is a heuristic for ensuring small $\kappa$ needed for CUR-like methods. Furthermore, we believe that performance could be improved if a more tuned way of pseudoinversion is used. Namely, as $Q^{(t)}$ is effectively rank-deficient for many epochs, it is crucial to implement a stable way of calculating the pseudoinverse of $L\tilde{Q}(\Ianchors,\Janchors)R$, and make it dependent on the current epoch and the level of the estimation error.

Lastly, we believe that the performance of the proposed methods can be significantly improved (compared to full-matrix methods) for larger state-action spaces, as well as by implementing a more advanced way of distributing samples across epochs.  

\vfill
\newpage
\section{Leverage Scores Estimation Analysis}\label{sec:proof-lscores}

In this section we provide the proof of Theorem \ref{thm:leverage_scores}. The proof relies on a tight instance dependent row-wise guarantee on the singular subspace recovery which is provided in Theorem \ref{thm:approx_sing_vectors_guarantee} together with a proof. Throughout this section and for brevity, we use the notation
\begin{align*}
    \Ttau = \frac{T}{\tau+1}, \qquad \mathrm{and} \qquad \balpha =  \frac{\Rmax}{1-\gamma}\frac{\sqrt{\cardS\cardA}}{\sigma_d(Q^\pi )}
\end{align*}
in the entirety of this section to denote the number of sampled trajectories $\Ttau$, and spikiness-related parameter $\balpha$ (recall definition of spikiness $\alpha$ from Section \ref{subsec:coherence}). Furthermore, recall the truncated value matrix $Q^\pi_\tau$ defined in Section \ref{subsec:prelim}, and let us define its corresponding approximation error by $\Delta = Q_\tau^\pi - Q^\pi$.

\subsection{Instance-dependent row-wise singular subspace recovery}
Below, we present Theorem \ref{thm:row-wise-guarantee}, which as highlighted before, is crucial in deriving Theorem \ref{thm:leverage_scores}. 

\begin{theorem}\label{thm:row-wise-guarantee}
    If $\Ttau = \widetilde{\Omega}\left( \balpha^2(\cardS+\cardA)\right)$ and $\Vert \Delta \Vert_\op \leq \sigma_{d}(Q^\pi)/32$, then we have that the event: for every $i \in [\cardS]$, $j \in [\cardA]$ 
    \begin{align*}
        \Vert U_{i,:} - \hU_{i,:} O_{\widehat{U}} \Vert_2 & = \widetilde{O} \left[ \balpha\left( \sqrt{\frac{d}{\Ttau}} + \kappa \Vert U_{i,:} \Vert_2 \sqrt{\frac{\cardS+\cardA}{\Ttau}} \right)  +\frac{ \sqrt{\cardS+\cardA}\Vert \Delta \Vert_{\infty}}{\sigma_{d}(Q^\pi)}   + \kappa\Vert U_{i,:} \Vert_2 \frac{\Vert \Delta\Vert_\op}{\sigma_d(Q^\pi)} \right], \\
           \Vert W_{j,:} - \hW_{j,:} O_{\widehat{W}} \Vert_2 & = \widetilde{O} \left[ \balpha\left( \sqrt{\frac{d}{\Ttau}} + \kappa \Vert W_{j,:} \Vert_2 \sqrt{\frac{\cardS+\cardA}{\Ttau}} \right)  +\frac{ \sqrt{\cardS+\cardA}\Vert \Delta \Vert_{\infty}}{\sigma_{d}(Q^\pi)}   + \kappa\Vert W_{j,:} \Vert_2 \frac{\Vert \Delta\Vert_\op}{\sigma_d(Q^\pi)} \right],
    \end{align*}
    holds with probability at least $1-\delta$, where we define  $O_{\widehat{U}} = \hU^\top U$ and $O_{\widehat{W}} = \hW^\top W$.
    \label{thm:approx_sing_vectors_guarantee}
\end{theorem}

\begin{corollary}
    If $\Vert \Delta \Vert_{\infty} \leq \min \left\{ \frac{\Rmax}{1-\gamma} \sqrt{\frac{d(S\wedge A)}{\Ttau}}, \frac{\sigma_d(Q^\pi)}{32\sqrt{SA}} \right\}$ and $\Ttau = \widetilde{\Omega}\left( \balpha^2(\cardS+\cardA)\right)$, then w.h.p:
    \begin{align*}
        \Vert U_{i,:} - \hU_{i,:}(\hU^\top U) \Vert_2 = \widetilde{O} \left[ \balpha \left( \sqrt{\frac{d}{\Ttau}} + \sqrt{\frac{\cardS+\cardA}{\Ttau}} \kappa \Vert U_{i,:} \Vert_2    \right)  \right].
    \end{align*}
    An analogous inequality holds for $\Vert W_{i,:} - \hW_{i,:}(\hW^\top W) \Vert_2$.
    \label{corr:perturbation_vectors_delta_vanished}
\end{corollary}
It is a simple algebraic exercise to show that $\epsilon = \Vert \Delta \Vert_{\infty}$ from \eqref{eq:eps_tau} satisfies condition of the corollary above in given regime of $\Ttau$.

\subsection{Proof of Theorem \ref{thm:leverage_scores}}
\label{app:proof_theorem_leverage}
\begin{proof}
    First we consider those states with $\Vert U_{s,:} \Vert_2^2 > \frac{d}{4S}$. From Corollary \ref{corr:perturbation_vectors_delta_vanished} we obtain that for these states and large enough $\Ttau$:
    \begin{align*}
        \Vert U_{s,:} - \hU_{s,:}(\hU^\top U) \Vert_2 \leq  c_1 \balpha \left( \sqrt{\frac{d}{\Ttau}} + \sqrt{\frac{\cardS+\cardA}{\Ttau}} \kappa \Vert U_{s,:} \Vert_2    \right) \log^{3/2} \left(\frac{\Ttau(\cardS+\cardA)}{\delta}\right)
    \end{align*}
    w.h.p. and for some universal constant $c_1 > 0$. Since $\Vert U_{s,:} \Vert_2 > \sqrt{\frac{d}{4S}}$ this implies that $\Vert U_{s,:} \Vert_2 > \sqrt{\frac{d}{4(S+A)\kappa^2}}$, we can simplify last inequality as follows:
   \begin{align*}
        \Vert U_{s,:} - \hU_{s,:}(\hU^\top U) \Vert_2 \leq  2 c_1 \balpha \sqrt{\frac{\cardS+\cardA}{\Ttau}} \kappa \Vert U_{s,:} \Vert_2  \log^{3/2} \left(\frac{\Ttau(\cardS+\cardA)}{\delta}\right)
    \end{align*}
    Next, for $\Ttau \geq 50 c_1^2 \balpha^2 (S+A) \kappa^2 \log^3\left(\frac{\Ttau(\cardS+\cardA)}{\delta}\right)$, we have: $\Vert U_{s,:} - \hU_{s,:}(\hU^\top U) \Vert_2 \leq ( 1- \frac{1}{\sqrt{2}}) \Vert U_{s,:} \Vert_2$. Finally, using reverse triangle inequality we obtain:
    \begin{align*}
        \Vert \hU_{s,:}\Vert_2^2 \geq ( \Vert U_{s,:}\Vert_2 - \Vert U_{s,:} - \hU_{s,:}(\hU^\top U) \Vert_2)^2 \geq \frac{1}{2}\Vert U_{s,:}\Vert_2^2
    \end{align*}
    and thus:
    \begin{align*}
        \tilde{\ell}_s =  \Vert \hU_{s,:} \Vert_2^2 \vee \frac{d}{S}  \geq  \Vert \hU_{s,:} \Vert_2^2 \geq \frac{1}{2}\Vert U_{s,:}\Vert_2^2
    \end{align*}
    
    Now we consider states with $\Vert U_{s,:} \Vert_2^2 \leq \frac{d}{4S}$. Again, by means of Corollary \ref{corr:perturbation_vectors_delta_vanished} we get w.h.p:
   \begin{align*}
        \Vert U_{s,:} - \hU_{s,:}(\hU^\top U) \Vert_2 \leq 2 c_1 \balpha \kappa \sqrt{\frac{d}{\Ttau}} \sqrt{\frac{S+A}{S}} \log^{3/2} \left(\frac{\Ttau(\cardS+\cardA)}{\delta}\right) \leq \sqrt{\frac{d}{4S}} 
    \end{align*}
    for $\Ttau\geq 16 c_1^2 \balpha^2 \kappa^2 (S+A)\log^3\left(\frac{\Ttau(\cardS+\cardA)}{\delta}\right)$. Thus we obtain that for all $s$ with $\Vert U_{s,:} \Vert_2^2 \leq \frac{d}{4S}$, it also holds:
    \begin{align*}
        \Vert \hU_{s,:} \Vert_2 \leq  \Vert U_{s,:} \Vert_2 + \Vert U_{s,:} - \hU_{s,:}(\hU^\top U) \Vert_2 \leq \sqrt{\frac{d}{S}} 
    \end{align*}
    Since by definition $\tilde{\ell}_s \geq \frac{d}{S}$, we obtain that $\tilde{\ell}_s \geq \Vert U_{s,:} \Vert_2^2$ for states with $\Vert U_{s,:} \Vert_2^2 \leq \frac{d}{4S}$.

    Finally, we show similar inequalities hold for leverage scores $\ell$ and $\hat{\ell}$. Namely, we have:
    \begin{align*}
        \hat{\ell}_s = \frac{\tilde{\ell}_s}{ \Vert \tilde{\ell} \Vert_1} &\geq \frac{\tilde{\ell}_s}{ \sum_{i: \Vert \hU_{i,:} \Vert_2^2 > \frac{d}{S}} \Vert \hU_{i,:} \Vert_2^2 + \sum_{j: \Vert \hU_{j,:} \Vert_2^2 \leq \frac{d}{S}} \frac{d}{S}}  \\ &\geq   \frac{\tilde{\ell}_s}{ \sum_{i =1}^S \Vert \hU_{i,:} \Vert_2^2 + S \frac{d}{S}} = \frac{\tilde{\ell}_s}{2d}   \geq \frac{\Vert U_{s,:}\Vert_2^2}{4d} = \frac{1}{4} \ell_s
    \end{align*}
    where we used first part of the proof for the final inequality.

\end{proof}

\subsection{Proof of Theorem \ref{thm:approx_sing_vectors_guarantee}}
\label{proof:row-wise-guarantee}

Proof is based on leave-one-out method used for proving entry.wise guarantees for singular vectors of SVD estimates. We refer the interested reader to \cite{chen2021spectral} for a comprehensive survey about the method. Here, we repeat the main arguments of the proof and improve the analysis in the following two ways:
\begin{itemize}
    \item[(i)] We keep track of approximation error during the whole proof in order to be able to apply it to approximately low rank matrix $Q_\tau^\pi$;
    \item[(ii)] Instead of showing guarantees in $\Vert \cdot \Vert_{2\to\infty}$ norm, we prove row/column specific guarantees. Indeed, note that our guarantees do not depend explicitly on the incoherence parameter $\mu$
    but instead the guarantee specific to the row vector $U_{i,:}$ of the singular subspace $U$,  depends only on its own incoherence parameter, i.e., $\Vert U_{i,:} \Vert_2$. This enables us to do leverage score analysis and obtain Theorem \ref{thm:leverage_scores}.  
\end{itemize}

Leave-one-out method is applied to symmetric matrices, so we first redefine our matrices in this context. For a matrix $Q^{\pi} \in \mathbb{R}^{\cardS \times \cardA}$ with SVD $Q^\pi = U\Sigma W^\top$ define symmetrizated matrix $\Md$ as follows:    \begin{align*}
        \Md =  \begin{bmatrix}
        0 & Q^\pi\\
        (Q^\pi)^\top & 0
    \end{bmatrix}    
     = \underbrace{\frac{1}{\sqrt{2}} \begin{bmatrix}
            U & U \\ W & -W
        \end{bmatrix}}_{\mathbf{U}}
        \underbrace{
        \begin{bmatrix}
            \Sigma & 0\\
            0 & -\Sigma
        \end{bmatrix}}_{\mathbf{\Sigma}}
        \underbrace{\frac{1}{\sqrt{2}} \begin{bmatrix}
            U& U \\ W & -W
        \end{bmatrix}^\top}_{\mathbf{U}^\top} 
    \end{align*}
    and similarly define symmetrized matrix $M$ from $Q_\tau^\pi$, $\bfDelta$ from $\Delta$, $\hM_d$ from $\hQ^\pi$, and $\bfE$ from $E = \tQ^\pi_\tau - Q^\pi_\tau$. Note that, using this notation, we have $M = \Md + \bfDelta $, with rank $d' = 2d$ matrix $\Md$ and $\Vert \bfDelta \Vert_{\infty} = \Vert \Delta \Vert_{\infty},\Vert \bfDelta \Vert_{\op} = \Vert \Delta \Vert_{\op} $. Assume observation matrix is given by $\tM = M + \bfE = \Md + \bfDelta + \bfE$ and note that $\hM_d = \bfhU \bfhSigma \bfhU^\top$ and $\hM_d = \tM \bfhU \bfhU^\top $. Thus, in order to prove lemma, it is sufficient to show: 
    \begin{align*}
        \Vert \bfU_{i,:} - \bfhU_{i,:}(\bfhU^\top \bfU) \Vert_2 \lesssim  \left( \sqrt{d'} + \kappa \Vert \bfU_{i,:} \Vert_2 \sqrt{\cardS+\cardA} \right)&\frac{\Rmax}{1-\gamma}\frac{\sqrt{SA}}{\sqrt{\Ttau}\sigma_{d'}(M)}\log^{3/2} \left(\frac{\Ttau(\cardS+\cardA)}{\delta}\right) \\
        & +\frac{ \sqrt{\cardS+\cardA}\Vert \bfDelta \Vert_{\infty}}{\sigma_{d'}(M)}  + \kappa\Vert \bfU_{i,:} \Vert_2 \frac{\Vert \bfDelta\Vert_\op}{\sigma_{d'}(M)} 
    \end{align*}

    For now, let us assume that $\Vert \bfE \Vert_\op \leq \sigma_{d'}(M)/32$, and prove that such inequality holds for $\Ttau$ large enough (consequence of Proposition \ref{prop:concentration}). 
    
    Define $\tM_d = \Md + \bfE$. Fix any $i \in [\cardS+\cardA]$ and for any matrix $A \in \mathbb{R}^{(\cardS+\cardA)\times n}$ define seminorm: $\Vert A\Vert_{2,i} = \Vert A_{i,:}\Vert_2$. Now using that $\bfU = \Md \bfU \bfSigma^{-1} = (\tM_d - \bfE) \bfU \bfSigma^{-1}$, we have:
    \begin{align*}
        \Vert \bfU - \bfhU \bfhU^\top \bfU \Vert_{2,i}  \le \frac{\Vert \tM_d \bfU - \bfhU (\bfhU^\top \bfU) \bfSigma \Vert_{2,i} }{\sigma_{d'}(M)} + \frac{\Vert \bfE \bfU \Vert_{2,i} }{\sigma_{d'}(M)} 
\end{align*}

To bound the numerator of the first term, 
note that $\bfhU (\bfhU^{\top} \bfU) \bfSigma = \bfhU \bfhU^{\top}\Md \bfU = \bfhU \bfhU^{\top} (\tM - \bfDelta - \bfE)  \bfU$. Since $\bfhU^{\top}\tM=\bfhU^{\top}\hM_d=\bfhSigma \bfhU^{\top},$
we get 
\begin{align*}
    \bfhU\bfhU^{\top} \bfU\bfSigma&=\bfhU\bfhSigma\bfhU^{\top} \bfU - \bfhU \bfhU^{\top} (\bfE+\bfDelta) \bfU \\
    &=\tM\bfhU \bfhU^{\top} \bfU - \bfhU \bfhU^{\top} (\bfE + \bfDelta) \bfU \\
    &=\tM_d\bfhU \bfhU^{\top} \bfU + \bfDelta\bfhU \bfhU^{\top} \bfU -\bfhU\bfhU^{\top} (\bfE + \bfDelta) \bfU \\
    &=\tM_d\bfhU \bfhU^{\top} \bfU + \bfDelta(\bfhU \bfhU^{\top} \bfU- \bfU) + \bfDelta \bfU -\bfhU \bfhU^{\top} (\bfE+\bfDelta) \bfU 
\end{align*}

Consequently, for $(\star) := \Vert \tM_d \bfU  - \bfhU (\bfhU^\top \bfU) \bfSigma \Vert_{2,i}$ we have:
   \begin{align*}
        (\star)
         \le 
         \Vert \Md (\bfU  -  \bfhU \bfhU^\top \bfU ) \Vert_{2,i}    &+  \Vert \bfE (\bfU  -  \bfhU \bfhU^\top \bfU ) \Vert_{2,i} + \Vert \bfDelta(\bfU-\bfhU \bfhU^{\top} \bfU)\Vert_{2,i} \\ &+ \Vert \bfDelta \bfU \Vert_{2,i} + \Vert 
         \bfhU \bfhU^\top (\bfE+\bfDelta) \bfU \Vert_{2,i}
    \end{align*}
    Throughout the proof we use Davis-Kahan's inequality (Corollary 2.8 in \cite{chen2021spectral}): if $\Vert \bfE \Vert_\op < (1-1/\sqrt{2})\sigma_{d'}(M)$, then:
    \begin{align*}
        \Vert \bfU - \bfhU \bfhU^\top \bfU \Vert_\op \leq \frac{2\Vert \bfE\Vert_\op}{\sigma_{d'}(M)}
    \end{align*}
    Now, similarly to (39) in \cite{stojanovic2024spectral} and using that $\Vert AB\Vert_{2,i} \leq \Vert A\Vert_{2,i} \Vert B\Vert_\op$ we obtain:
    \begin{align*}
        \Vert \Md (\bfU - \bfhU \bfhU^\top \bfU)\Vert_{2,i} &\leq 4\Vert \bfU \Vert_{2,i} \Vert \bfSigma \Vert_\op \frac{ \Vert \bfE \Vert_\op^2  }{\sigma_{d'}^2(M)} \\
        \Vert \bfDelta ( \bfU - \bfhU \bfhU^\top \bfU)\Vert_{2,i} &\leq 2\Vert \bfDelta \Vert_{2,i} \frac{ \Vert \bfE \Vert_\op  }{\sigma_{d'}(M)}\\
        \Vert \bfDelta \bfU\Vert_{2,i} &\leq \Vert \bfDelta \Vert_{2,i} \Vert \bfU \Vert_\op \leq\Vert \bfDelta \Vert_{2,i}
    \end{align*}

Using these inequalities together with Lemma \ref{lemma:technical_term_LOO_analysis} and $\frac{\Vert \bfE \Vert_\op + \Vert \bfDelta \Vert_\op}{\sigma_{d'}(M)} \leq 1/16$ we obtain:
      \begin{align*}
        (\star) 
        \le  2\Vert \bfDelta \Vert_{2,i} +  \frac{9}{2}\Vert \bfU \Vert_{2,i}  \frac{\Vert \bfSigma \Vert_\op}{\sigma_{d'}(M)} (\Vert \bfE\Vert_\op + \Vert \bfDelta \Vert_\op)  +  2\Vert \bfE (\bfU  -  \bfhU \bfhU^\top \bfU ) \Vert_{2,i} + \Vert \bfE  \bfU \Vert_{2,i} 
    \end{align*}
    which in the end gives 
       \begin{align}
        \Vert \bfU - \bfhU \bfhU^\top \bfU \Vert_{2,i} \le \frac{2}{\sigma_{d'}(M)} \Big(5 \Vert \bfU \Vert_{2,i} & \frac{\Vert \bfSigma \Vert_\op}{\sigma_{d'}(M)} (\Vert \bfE\Vert_\op + \Vert \bfDelta \Vert_\op) + \Vert \bfE \bfU\Vert_{2,i} \nonumber \\ &+  \Vert \bfE (\bfU - \bfhU \bfhU^\top  \bfU )\Vert_{2,i} + \Vert \bfDelta \Vert_{2,i} \bigg)
        \label{eq:Uperp_ineq_midstep}
    \end{align}
    \paragraph{Leave-one-out decomposition:}
    Define matrix $\tMi$ as follows:
   \begin{align*}
        \tMi_{j,k} = \begin{cases}
            \tM_{j,k},\quad &\text{if } j\neq i \text{ or } k \neq i \\
            M_{j,k}, \quad &\text{otherwise}
        \end{cases}
    \end{align*}
    and let $\bfhUi$ denote matrix of $d'$ dominant singular vectors of $\tMi$. Then, by triangle inequality we can write:
    \begin{align*}
        \Vert \bfE (\bfU - \bfhU \bfhU^\top  \bfU )\Vert_{2,i} \leq \Vert \bfE (\bfU - \bfhUi (\bfhUi)^\top  \bfU )\Vert_{2,i} + \Vert \bfE \Vert_\op \Vert \bfhUi (\bfhUi)^\top \bfU - \bfhU \bfhU^\top  \bfU \Vert_{\F}
    \end{align*}
    We bound the last term using Lemma \ref{lemma:LOO_analysis_frobenius} to obtain:
     \begin{align}
        \Vert \bfE (\bfU - \bfhU \bfhU^\top  \bfU )\Vert_{2,i} \leq 2\Vert \bfE &(\bfU - \bfhUi (\bfhUi)^\top  \bfU )\Vert_{2,i} \nonumber \\&+ 
        6\frac{\Vert \bfE \Vert_\op}{\sigma_{d'}(M)} \bigg( \Vert \bfE \bfU \Vert_{2,i} + 2\Vert \bfE\Vert_\op (\Vert \bfU \Vert_{2,i} + \Vert \bfU - \bfhU (\bfhU^\top \bfU) \Vert_{2,i}) \bigg)
        \label{eq:EUpert_midstep}
    \end{align}
    Substituting \eqref{eq:EUpert_midstep} into \eqref{eq:Uperp_ineq_midstep} we obtain:
    \begin{align*}
         \Vert \bfU - \bfhU \bfhU^\top \bfU \Vert_{2,i} \le \frac{12}{\sigma_{d'}(M)} \Big( \Vert \bfU \Vert_{2,i} & \frac{\Vert \bfSigma \Vert_\op}{\sigma_{d'}(M)} (\Vert \bfE\Vert_\op + \Vert \bfDelta \Vert_\op)  + \Vert \bfE \bfU\Vert_{2,i}  \\ &+  \Vert \bfE (\bfU - \bfhUi (\bfhUi)^\top  \bfU )\Vert_{2,i}  + \Vert \bfDelta \Vert_{2,i} \bigg)
    \end{align*}
    Then we apply Proposition \ref{prop:concentration} on $\Vert \bfE \bfU\Vert_{2,i}$ and $\Vert \bfE (\bfU - \bfhUi (\bfhUi)^\top  \bfU )\Vert_{2,i}$. We use that $\Vert \bfU \Vert_{\F} \leq \sqrt{d'}$ and $\Vert \bfU - \bfhUi (\bfhUi)^\top \bfU \Vert_{\F} \leq 2\sqrt{d'}$. Finally, we use that $\Vert \bfDelta \Vert_{2,i} \leq \sqrt{S+A} \Vert \Delta \Vert_{\infty}$, the fact that
    \begin{align*}
        \sigma_{d'}(M) \geq \sigma_{d'}(\Md) - \Vert \bfDelta \Vert_\op \geq \sigma_{d'}(\Md)(1-1/32) 
    \end{align*}
    and thus $\frac{\Vert \bfSigma\Vert_\op}{\sigma_{d'}(M)} \leq 2\frac{\Vert \Md \Vert_\op}{\sigma_{d'}(\Md)} \leq 2\kappa$.

\subsection{Rank estimation guarantee}
\label{app:subsec_rank_estimation}
Recall from Section \ref{subsec:l-scores-est} and Proposition \ref{thm:approx_sing_vectors_guarantee_informal} that, given the singular values $(\hsigma_i)_i$ of our estimates $\tQ_\tau^\pi$, we estimate effective rank as follows: $\hat{d} = \sum_{i=1}^{S \wedge A} \indicator{\lbrace \hat{\sigma}_i \ge \beta \rbrace}$ with
\begin{align*}
    \beta =  \sqrt{\frac{ r_{\max}^2 \cardS \cardA (S+ A)}{ (1-\gamma)^3 T} \log^4\left( \frac{  (S+A)T}{ (1-\gamma)
    \delta}\right) }    + \frac{\Rmax\sqrt{SA}}{T} 
\end{align*}

Here we repeat first part of the Proposition \ref{thm:approx_sing_vectors_guarantee_informal} and prove it:\begin{lemma}    
If $T$ satisfies \eqref{eq:T_requirement_beta}, then estimated rank $\hat{d}$ satisfies $\hat{d} = d_\pi$ with probability at least $1-\delta$.
\label{lemma:rank_estimation_appendix}
\end{lemma}
\begin{proof}
    By our assumptions $\sigma_i(Q^\pi) > 0$ only for $i\in[d_\pi]$, and thus $\forall i> d_\pi:\
        \hsigma_i \leq \Vert \tQ_\tau^\pi - Q^\pi \Vert_\op$ and $\forall i\leq d_\pi:\
        \hsigma_i \geq \sigma_{d_\pi}(Q^\pi) - \Vert \tQ_\tau^\pi - Q^\pi \Vert_\op$. Recall that $E = \tQ^\pi_\tau - Q^\pi_\tau$ and $\Delta = Q_\tau^\pi - Q^\pi$, and thus:
        \begin{align*}
            \Vert \tQ_\tau^\pi - Q^\pi \Vert_\op \leq \Vert E\Vert_\op + \Vert \Delta \Vert_\op
        \end{align*}
    We bound the second term by $\Vert \Delta \Vert_\op \leq \sqrt{\cardS \cardA}\Vert \Delta \Vert_{\infty}$ and use that  $\Vert \Delta \Vert_{\infty} \leq \frac{\Rmax}{T}$ from Lemma \ref{lem:truncated-Q}. The first term is upper bounded by Proposition \ref{prop:concentration}. Combining the two, we obtain that
    $\Vert \tQ_\tau^\pi - Q^\pi \Vert_\op \leq \beta$ with high probability. Thus, for $2\beta \leq \sigma_{d_\pi}(Q^\pi)$ we are guaranteed to recover the true rank $d_\pi$, since then $\forall i\in[d_\pi]$:
    \begin{align*}
        \hsigma_i \geq \sigma_{d_\pi}(Q^\pi) - \Vert \tQ_\tau^\pi - Q^\pi \Vert_\op \geq 2\beta - \beta \geq \beta
    \end{align*}

    It is straightforward to check that, if
     \begin{align*}
          T = {\Omega}\left(\frac{ r_{\max}^2  \cardS\cardA (\cardS+\cardA) } { (1-\gamma)^3  \sigma_{d}^2(Q^\pi)} \log^4\left( \frac{  (S+A)T}{ (1-\gamma)
    \delta}\right) \right) 
      \end{align*}
      then first term in definition of $\beta$ is $\leq \sigma_{d_\pi}(Q^\pi)/4$, and similar conclusion hold for the second term after noting that $\frac{\Rmax\sqrt{SA}}{\sigma_{d_\pi}(Q^\pi)}(S+A)\geq 1$.
\end{proof}

    \subsection{Technical lemmas from the proof of Theorem \ref{thm:approx_sing_vectors_guarantee}}
    In this section we shortly present concentration results used in the proof of Theorem \ref{thm:approx_sing_vectors_guarantee}. We refer reader to Section \ref{app:noise_equivalence} for discussion about equivalent noise model and Poisson approximation. Concentration inequalities proposed are fairly standard (see for example \cite{chen2021spectral}), but as discussed in \cite{stojanovic2024spectral} because of the sampling pattern, entries of the matrix $\bfE$ are slightly dependent. A way to deal with these dependencies has been discussed in \cite{stojanovic2024spectral}, and we refer to the results from that paper here directly. 
    
    \begin{proposition}
      Let $B$ be a $(\cardS+\cardA) \times 2d$ matrix independent of $\bfE$. Then, we have for all $\delta\in (0,1)$, for all $
            \Ttau \gtrsim (\cardS+\cardA) \log^3\left( (\cardS+\cardA)/\delta\right)$, both events:
            \begin{align*}
                 \Vert \bfE \Vert_\op  \lesssim  \frac{\Rmax}{1-\gamma}  \sqrt{\frac{\cardS \cardA}{\Ttau}} \sqrt{ \cardS+\cardA }   \log^{3/2} \left(\frac{\Ttau(\cardS+\cardA)}{\delta}\right)  
            \end{align*}  
            \begin{align*}
                \forall i \in [\cardS+\cardA]:\quad 
                \Vert \bfE_{i,:} B \Vert_\op
                \lesssim \frac{\Rmax}{1-\gamma} \Vert B \Vert_{\F} \sqrt{\frac{\cardS \cardA}{\Ttau}} \log^{3/2}\left( \frac{\Ttau(\cardS+\cardA)}{\delta}\right)  
            \end{align*}
            hold with probability at least $1-\delta$.  
            \label{prop:concentration}
    \end{proposition}
    \begin{proof}
    Follows straightforwardly from proofs of Propositions 26 and 27 in \citep{stojanovic2024spectral} and using noise equivalent model from Section \ref{app:noise_equivalence}. Note that we keep the variance term upper bounded by $\Vert A\Vert_\F^2$ in the proof of Proposition 27 as in \cite{stojanovic2024spectral}, and make use of inequality $\Vert B\Vert_{2\to\infty} \leq \Vert B\Vert_{\F}$ to obtain dependence on $\Vert B\Vert_\F$ in the second inequality.
    \end{proof}
    
    \begin{lemma}
    If $\Vert \bfE \Vert \le \sigma_{d'}(M)/32$, then for every $i$:
        \begin{align*}
             \Vert \bfhU \bfhU^\top (\bfE+\bfDelta) \bfU \Vert_{2,i} \leq 4\frac{\Vert \bfE \Vert_\op + \Vert \bfDelta \Vert_\op}{\sigma_{d'}(M)} \Big(  \Vert \bfU \Vert_{2,i} \Vert \bfSigma \Vert_\op &+ \Vert \bfDelta \Vert_{2,i} \\ & + \Vert \bfE  \bfU \Vert_{2,i}  +  \Vert \tM ( \bfU - \bfhU \bfhU^\top \bfU)\Vert_{2,i} \Big)
        \end{align*}
        \label{lemma:technical_term_LOO_analysis}
    \end{lemma}
    \begin{proof}
    Under condition $\Vert \bfE \Vert \le \sigma_{d'}(M)/32$ we have $\sigma_{d'}(\tM) \ge \sigma_{d'}(M) - \Vert \bfE \Vert \ge  \sigma_{d'}(M)/2$. Thus, we have:
    \begin{align}
        \Vert \bfhU \bfhU^\top (\bfE + \bfDelta) \bfU \Vert_{2,i} &= \Vert \hM_d \bfhU \bfhSigma^{-1} \bfhU^\top (\bfE + \bfDelta) \bfU \Vert_{2,i} = \Vert \tM\bfhU \bfhSigma^{-1} \bfhU^\top (\bfE + \bfDelta) \bfU\Vert_{2,i} \nonumber \\
        &\leq \Vert \tM\bfhU\Vert_{2,i} \Vert \bfhSigma^{-1} \Vert_\op \Vert \bfhU^\top \Vert_\op \Vert \bfE + \bfDelta \Vert_\op \Vert \bfU\Vert_\op \leq \Vert \tM\bfhU \Vert_{2,i} \frac{\Vert \bfE + \bfDelta \Vert_\op}{\sigma_{d'}(\tM)} \nonumber \\ 
        &\le 2\frac{  \Vert \bfE \Vert_\op + \Vert \bfDelta \Vert_\op  }{\sigma_{d'}(M)} \Vert \tM \bfhU \Vert_{2,i}
        \label{eq:1}
    \end{align}

    Using Davis-Kahan's inequality \cite{chen2021spectral} we have:
    \begin{align*}
        \Vert \tM \bfhU \Vert_{2,i} & = \Vert \tM \bfhU \sgn(\bfhU^\top  \bfU) \Vert_{2,i} \\ 
        & \le \Vert \tM \bfhU \bfhU^\top \bfU \Vert_{2,i} +  \Vert \tM \bfhU \Vert_{2,i}  
        \Vert \sgn(\bfhU^\top  \bfU) - \bfhU^\top \bfU \Vert_\op \\
        & \le \Vert \tM \bfhU \bfhU^\top \bfU \Vert_{2,i} +    16 \Vert \tM \bfhU \Vert_{2,i}\frac{ \Vert \bfE \Vert^2 }{\sigma_{d'}(M)^2} \\
        & \le \Vert \tM \bfhU \bfhU^\top \bfU \Vert_{2,i} +    \frac{\Vert \tM \bfhU \Vert_{2,i} }{2} 
    \end{align*}
    implying that $\Vert \tM \bfhU \Vert_{2,i} \leq 2 \Vert \tM \bfhU \bfhU^\top \bfU \Vert_{2,i}$. Furthermore, we have:
    \begin{align*}
    \Vert \tM \bfhU \bfhU^\top \bfU \Vert_{2,i}
        &\leq \Vert \tM  \bfU \Vert_{2,i} + \Vert \tM ( \bfU - \bfhU \bfhU^\top \bfU)\Vert_{2,i} \\
        &\leq \Vert \bfU \Vert_{2,i}\Vert \bfSigma \Vert_\op  +\Vert (\bfE+\bfDelta)  \bfU \Vert_{2,i} + \Vert \tM ( \bfU - \bfhU \bfhU^\top \bfU)\Vert_{2,i} 
    \end{align*}
    where we have used that $M \bfU = \bfU\bfSigma$. Substituting this expression back into \eqref{eq:1} finishes the proof.
    \end{proof}

\begin{lemma}
Under assumptions $\Vert \bfDelta \Vert_\op \leq \sigma_{d'}(M)/32$ and $\Vert \bfE \Vert_\op \leq \sigma_{d'}(M)/32$, we have with high probability for every $i$:
       \begin{align*}
        \Vert \bfhUi (\bfhUi)^\top \bfU - \bfhU \bfhU^\top  \bfU \Vert_{\F} \leq \frac{6}{\sigma_{d'}(M)} \bigg( \Vert \bfE \bfU \Vert_{2,i} &+ \Vert \bfE (\bfU - \bfhUi (\bfhUi)^\top  \bfU )\Vert_{2,i} \\ &+ 2\Vert \bfE\Vert_\op (\Vert \bfU \Vert_{2,i} + \Vert \bfU - \bfhU (\bfhU^\top \bfU) \Vert_{2,i}) \bigg)
    \end{align*}
    \label{lemma:LOO_analysis_frobenius}
\end{lemma}
\begin{proof}
    Proof follows similar steps as Step 2.2 in the proof of Theorem 4.2 in \citep{chen2021spectral}, but we repeat it here for the sake of completeness and focus on differences in the proof caused by $\bfDelta$ matrix. First, we use that $\bfU$ is an orthogonal matrix ($\Vert \bfU\Vert_\op = 1$) and Davis-Kahan's inequality to obtain:
    \begin{align*}
        \Vert \bfhUi (\bfhUi)^\top \bfU - \bfhU \bfhU^\top  \bfU \Vert_\F \leq \Vert \bfhUi (\bfhUi)^\top  - \bfhU \bfhU^\top  \Vert_\F \leq 2 \frac{\Vert (\tM - \tMi )\bfhUi \Vert_\F}{\sigma_{d'}(\tMi) - \sigma_{d'+1}(\tMi)}
    \end{align*}
    Note that under the assumption $\Vert \bfDelta \Vert_\op \leq \sigma_{d'}(M) $ analysis of singular values stays the same as in \cite{chen2021spectral}, since, for example:
    \begin{align*}
        \sigma_{d'}(\tMi) &\geq \sigma_{d'}(M) - \Vert \bfEi \Vert_\op \geq \sigma_{d'}(M) \left( 1- \frac{1}{32} \right) \\
        \sigma_{d'+1}(\tMi) &\leq \sigma_{d'+1}(M) + \Vert \bfEi \Vert_\op \leq \Vert \bfDelta \Vert_\op + \Vert\bfE \Vert_\op \leq \sigma_{d'}(M)/16
    \end{align*}
    Thus, we can lower bound denominator in the inequality above by $\sigma_{d'}(M)/2$. Now, term in the numerator can be written as:
    \begin{align*}
        (\tM - \tMi )\bfhUi = \bfE_{i,:} \bfhUi + (\bfE_{:,i} - \bfE_{i,i}e_i) \bfhUi_{i,:}
    \end{align*}
    and bounded in the same way as in \cite{chen2021spectral}:
    \begin{align*}
        \Vert (\tM - \tMi )\bfhUi \Vert_\F &\leq \Vert \bfE \bfhUi \Vert_{2,i} + \Vert \bfE_{:,i} - \bfE_{i,i}e_i \Vert_2 \Vert \bfhUi\Vert_{2,i}\\
        & \leq \Vert \bfE \bfhUi \Vert_{2,i} + 2\Vert \bfE \Vert_\op \Vert \bfhUi ((\bfhUi)^\top \bfU)\Vert_{2,i}
    \end{align*}
    where we used that $\Vert ((\bfhUi)^\top \bfU)^{-1} \Vert_2 \leq 2$ under our assumptions. Finally, we obtain:
    \begin{align*}
        \Vert \bfhU\bfhU^\top \bfU - \bfhUi (\bfhUi)^\top \bfU\Vert_\F \leq \frac{4}{\sigma_{d'}(M)} ( \Vert \bfE \bfhUi \Vert_{2,i} &+ 2\Vert \bfE\Vert_\op \Vert \bfhU (\bfhU^\top \bfU) \Vert_{2,i} \\ &+ 2 \Vert \bfE\Vert_\op \Vert \bfhU\bfhU^\top \bfU -  \bfhUi (\bfhUi)^\top \bfU\Vert_\F )
    \end{align*}
    and under condition $\Vert \bfE\Vert_\op \leq \sigma_{d'}(M)/32$ we obtain result claimed in the lemma.
\end{proof}

\subsection{Nuclear norm minimization for leverage scores estimation}
   The authors of \cite{sam2023overcoming} leveraged the guarantees for nuclear norm minimization from \cite{chen2020noisy} to learn $Q$ matrices. However, several factors make the application of nuclear norm minimization theoretically challenging in our context:
   \begin{itemize}
        \item \textbf{Approximate low-rank structure.} As our algorithm is based on policy iteration, our estimates $Q_{\tau}^\pi$ are only approximately low-rank. The result from \cite{chen2020noisy} rely on non-convex optimization, leaving it unclear how this approximation error affects the final guarantees of the algorithm. In contrast, a more straightforward analysis using singular value decomposition allows us to explicitly express our bounds in terms of the approximation error. 
        
        \item \textbf{Coherence-free subspace recovery.} In our subspace recovery result (Theorem \ref{thm:row-wise-guarantee}), we can bound the subspace error $\Vert U_{i,:} - \widehat{U}_{i,:} O_{\widehat{U}}\Vert_{2}$ in relation to $\Vert U_{i,:}\Vert_2$. It is uncertain whether current guarantees for nuclear norm minimization can achieve a similar outcome, which might instead depend on $\max_{i\in [S]} \Vert U_{i,:}\Vert_2$. We believe this would introduce dependency on the incoherence constant into the sample complexity of our algorithm.  
\end{itemize}

\newpage
\section{Leveraged Matrix Estimation Analysis}
\label{app:proof_CUR_theorem}

In this appendix, we provide the proof of Theorem \ref{thm:lme-guarantee} corresponding to sample complexity guarantee enjoyed by  $\lme$. First, we provide the pseudo-code of $\lme$:

\begin{algorithm}[h!]
    \SetAlgoLined
    \KwIn{Deterministic policy $\pi$, sampling budget $T$}
    Set $T_1 \gets T/2$, $T_2 \gets T/2$ \\
    Set $\epsilon = \frac{\Rmax}{T}$, $\tau \gets \frac{1}{(1-\gamma)} \log\left(\frac{T}{1-\gamma}\right)$ as in \eqref{eq:eps_tau}\\
    \emph{\color{purple} \underline{(Phase 1). Leverage Scores Estimation:}}\\
    \ \ \ \ \emph{\color{purple}(Phase 1a.) Data Collection $\&$ Emprirical Truncated Value matrix.}\\
    \ \ \ \ Sample uniformly at random $T_1/(\tau + 1)$ trajectories of length $\tau+1$ using policy $\pi$ and\\  \ \ \ \ \ \ gather them in  $\cD$ \\ 
    \ \ \ \ Use the collected dataset $\cD$, to construct $\widetilde{Q}^\pi_\tau$ as in \eqref{eq:emp-truncated-q} \\
    \ \ \ \ \emph{\color{purple}(Phase 1b.) Singular Subspace Recovery}\\
    \ \ \ \ Set the threshold $\beta$ as in \eqref{eq:threshold} \\
    \ \ \ \ Compute the SVD of $\widetilde{Q}^\pi_\tau$ and threshold with $\beta$ as described in \eqref{eq:est-svt} to obtain $\widehat{d}$, $\hU$, $\hW$ and $\widehat{Q}^\pi$\\
    \ \ \ \ \emph{\color{purple}(Phase 1c.) Leverage Scores.}\\
    \ \ \ \ Set the left (resp. right) leverage scores $\hat{\ell}$ ( resp. $\hat{\rho}$) as described in \eqref{eq:est-lev}. \\

    \emph{\color{purple}\underline{(Phase 2.) CUR-based Matrix Estimation with Leverage.}}\\
    \ \ \ \ \emph{\color{purple}(Phase 2a.) Data Collection with Leverage \& Empirical Truncated Value Matrix:}\\
    \ \ \ \ Set $K  \gets 64\hat{d} \log(64\hat{d}/\delta)$ \\
    \ \ \ \ Sample $K$ rows (resp. columns) $\cI \subset \cS$ (resp. $\cJ \subset \cA$) without replacement according to the \\ 
    \ \ \ \ \ \ leverage scores $\hat{\ell}$ (resp. $\hat{\rho}$) \\
    \ \ \ \ Set $N_1 \gets \frac{T_2}{2 (\tau+1)K^2}$, $N_2 \gets \frac{T_2}{2 (\tau+1)(K(S+A) - 2 K^2)} $  \\
    \ \ \ \ For all $(s,a) \in \Omega_\square$ (resp. $(s,a) \in \Omega_+$) sample $N_1$ (resp. $N_2$) trajectories of length $\tau+1$ using policy $\pi$ and construct the \\ 
    \ \ \ \ \ \ empirical estimate $\widetilde{Q}_\tau^\pi(s,a)$ based on these trajectories \\
    \ \ \ \ \emph{\color{purple}(Phase 2b. CUR-based Matrix estimation)} \\ 
    \ \ \ \ Set the matrices $L$ and $R$ as in \eqref{eq:LR-weighting} \\
    \ \ \ \ Construct $\widehat{Q}^\pi$ using a CUR-based approach as in 
    \eqref{eq:CUR-ME} \\
    \KwOut{$\widehat{Q}^\pi$.}
    \caption{Leverage Matrix Estimation (\lme)}
\end{algorithm}

\subsection{Proof of Theorem \ref{thm:lme-guarantee}} \label{subsec:proof-lme-gurantee}

Before showing the proof of Theorem \ref{thm:lme-guarantee} we present two intermediate results used in the proof. As an immediate consequence of Hoeffding's inequality we have: 
\begin{lemma}
With probability at least $1-\delta$ we have $\forall (s,a) \in (\Ianchors \times \actions) \cup (\states \times \Janchors)$:
    \begin{align*}
        \vert \tQ_\tau^\pi (s,a) - Q^\pi(s,a) \vert \leq \frac{\Rmax}{1-\gamma} \sqrt{\frac{2}{N}\log\left( \frac{4 \nachors (\cardS + \cardA)}{\delta} \right)} + \Vert Q_\tau^\pi - Q^\pi \Vert_{\infty} 
    \end{align*}
    where $N=N_1$ if $(s,a)\in \Omega_\square$ or $N=N_2$ if $(s,a)\in \Omega_+$.
    \label{lemma:Hoeffding_cross_entries}
\end{lemma}

\begin{theorem} 
Let $\Ianchors$ and $\Janchors$ be such that $\vert \Ianchors \vert , \vert \Janchors \vert = \nachors$, and $Q^\pi(\Ianchors, \Janchors)$ has rank $d$. Assume that for some $\varepsilon_\square,\varepsilon_+>0$:
\begin{enumerate}[label=\alph*)]
    \item $\forall (s,a)\in \Ianchors\times \Janchors: \vert \tQ_\tau^\pi(s,a) - Q^\pi(s,a)\vert \leq \varepsilon_\square$, and
    \item $\forall (s,a)\in(\Ianchors\times \actions)\cup (\states \times \Janchors)\setminus (\Ianchors \times \Janchors): \vert \tQ_\tau^\pi(s,a) - Q^\pi(s,a)\vert \leq \varepsilon_+$.
\end{enumerate}
If $\varepsilon_\square \leq \frac{1}{8c_\Ianchors  c_\Janchors } \frac{\sigma_d(Q^\pi)}{\sqrt{\cardS \cardA}} \log^{-2}\left( \frac{\cardS + \cardA}{\delta} \right)$, $\varepsilon_+ \leq \Vert Q^\pi\Vert_{\infty}$ and $\nachors \geq 64d\log(64d/\delta)$, then with probability $\geq 1-\delta$:
\begin{align*}
    \Vert \hQ^\pi - Q^\pi \Vert_{\infty} \lesssim    \Vert Q^\pi \Vert_{\infty}\varepsilon_+  \frac{\sqrt{\cardS\cardA}}{\sigma_d(Q^\pi) } \log^2 \left( \frac{S+A}{\delta} \right)   +   \Vert Q^\pi \Vert_{\infty}^2 \varepsilon_\square \frac{\cardS\cardA}{\sigma_d^2(Q^\pi)}\log^4 \left( \frac{S+A}{\delta} \right)   
\end{align*}
\label{thm:cross_Qmax_guarantee}
\end{theorem}
Proof of Theorem \ref{thm:cross_Qmax_guarantee} is deferred to \ref{subsec:proof_cross_Qmax_guarantee}.

\paragraph{Proof of Theorem \ref{thm:lme-guarantee}.} 
First, by Theorem \ref{thm:leverage_scores} we require at least 
\begin{align}
    T \gtrsim \frac{r_{\max}^2}{(1-\gamma)^3\Vert Q^\pi \Vert_\infty^2} \kappa^2 (S + A) \frac{\Vert Q^\pi \Vert_\infty^2 SA}{\sigma_{d}^2(Q^\pi) }     \log^4\left( \frac{ (S+A)T}{(1-\gamma)
    \delta}\right)
    \label{eq:T_cond_phase_I_proof_Thm2}
\end{align}
number of samples to recover leverage scores of $Q^\pi$. During the whole proof of Theorem \ref{thm:lme-guarantee} we condition on the event where Theorem \ref{thm:leverage_scores} holds.

Recall that we use $N_1 = \frac{T}{4 (\tau+1)K^2}$, $N_2 = \frac{T}{4 (\tau+1)(K(S+A) - 2 K^2)} $ and define the following quantities:
\begin{align*}
    \varepsilon_\square = \frac{\Rmax}{1-\gamma} \sqrt{\frac{8}{N_1}\log\left( \frac{4 \nachors (\cardS + \cardA)}{\delta} \right)},\qquad
    \varepsilon_+ = \frac{\Rmax}{1-\gamma} \sqrt{\frac{8}{N_2}\log\left( \frac{4 \nachors (\cardS + \cardA)}{\delta} \right)}.
\end{align*}
Note that by definitions of $N_1$ and $N_2$ we have that $\varepsilon_\square = \varepsilon_+ \sqrt{\frac{\nachors }{S+A-2\nachors}}$. Next, recall that $\Vert Q_\tau^\pi - Q^\pi \Vert_{\infty}$ is upper-bounded by $\frac{r_{\max} }{T}$ from \eqref{eq:eps_tau}. Combining this with Lemma \ref{lemma:Hoeffding_cross_entries} and our definition of $\varepsilon_\square, \varepsilon_+$ we can see that the conditions $a)$ and $b)$ of Theorem \ref{thm:cross_Qmax_guarantee} are met.

Hence, by Theorem \ref{thm:cross_Qmax_guarantee} we obtain that $\Vert \hQ^\pi - Q^\pi \Vert_{\infty} \leq \varepsilon$ if:
\begin{align*}
    \varepsilon_+ \lesssim \frac{\varepsilon}{ \Vert Q^\pi\Vert_{\infty}\frac{ \sqrt{\cardS\cardA}}{\sigma_d(Q^\pi)}\log^2 \left( \frac{\cardS + \cardA}{\delta} \right) \left(1 +  \Vert Q^\pi\Vert_{\infty}\frac{ \sqrt{\cardS\cardA}}{\sigma_d(Q^\pi)}\log^2 \left( \frac{\cardS + \cardA}{\delta} \right) \sqrt{\frac{\nachors}{S+A-2\nachors}}\right)}
\end{align*}
Setting $\varepsilon_+$ to be equal to this value we obtain that is sufficient to have:
\begin{align*}
    N_2 \gtrsim \frac{\Rmax^2 }{\varepsilon^2(1-\gamma)^2} \frac{ \Vert Q^\pi \Vert_{\infty}^2 \cardS\cardA}{\sigma_d^2(Q^\pi)} \left( 1 + \frac{\Vert Q^\pi \Vert_{\infty}^2 \cardS\cardA}{\sigma_d^2(Q^\pi)} \frac{\nachors}{S+A-2\nachors} \log^4 \left( \frac{\cardS + \cardA}{\delta} \right) \right) \log^5\left( \frac{d(\cardS + \cardA)}{\delta} \right) 
\end{align*}

Using definition of $N_2$ and rewriting inequality above in terms of $T$ gives the following condition:
\begin{align*}
    T \gtrsim \frac{r_{\max}^2 d}{(1-\gamma)^3 \varepsilon^2} \frac{\Vert Q^\pi \Vert_{\infty}^2 \cardS\cardA}{\sigma_d^2(Q^\pi)} \left( S+A + d \frac{\Vert Q^\pi \Vert_{\infty}^2 \cardS\cardA}{\sigma_d^2(Q^\pi)} \log^5 \left( \frac{S+A}{\delta} \right) \right) \log^7 \left( \frac{(S+A)T}{\delta(1-\gamma)} \right)
\end{align*}
Combining this with condition \eqref{eq:T_cond_phase_I_proof_Thm2} and using the fact that $\frac{\Vert Q^\pi \Vert_{\infty}^2 SA}{\sigma_d^2(Q^\pi)} \leq \kappa^2 \alpha^2 d$ gives the final condition:
\begin{align*}
    T = \widetilde{\Omega}_\delta \left[ r_{\max}^2 \kappa^2 \alpha^2 \frac{d(S+A)}{(1-\gamma)^3} \left( \frac{\kappa^2}{\Vert Q^\pi \Vert_\infty^2} + \frac{d}{\varepsilon^2} + \frac{d^2\alpha^2\kappa^2}{(S+A)\varepsilon^2}\right)\right]
\end{align*}

Finally we verify that $\varepsilon_\square$ and $\varepsilon_+$ satisfy conditions from Theorem \ref{thm:cross_Qmax_guarantee}. Note that $\Vert Q^\pi\Vert_{\infty}\frac{ \sqrt{\cardS\cardA}}{\sigma_d(Q)} \geq \Vert Q^\pi\Vert_{\infty}\frac{ \sqrt{\cardS\cardA}}{\Vert Q^\pi \Vert_\F} \geq 1$, as well as $\log^2 \left( \frac{\cardS + \cardA}{\delta} \right) \geq 1$ for any $\delta \in (0,1)$. Thus we obtain that $\varepsilon_+ \lesssim \varepsilon$ and thus, in order to have $\varepsilon_+ \lesssim \Vert Q^\pi \Vert_{\infty}$ it is sufficient to assume that $\varepsilon \lesssim \Vert Q^\pi \Vert_{\infty}$. Using the same reasoning we have:
\begin{align*}
    \varepsilon_\square \lesssim \varepsilon_+ \sqrt{\frac{\nachors}{S+A}} \lesssim \varepsilon \frac{\sigma_d(Q^\pi)}{\Vert Q^\pi \Vert_\infty \sqrt{SA}\log^2\left( \frac{S+A}{\delta} \right)} \sqrt{\frac{\nachors}{S+A}} \lesssim \frac{\sigma_d(Q^\pi)}{\sqrt{SA}}\log^{-2}\left(\frac{S+A}{\delta} \right)
\end{align*}
whenever $\varepsilon \lesssim \Vert Q^\pi \Vert_\infty$ and $S+A \gtrsim \nachors$.

\subsection{Proof of Theorem \ref{thm:cross_Qmax_guarantee}}
\label{subsec:proof_cross_Qmax_guarantee}
Note that $(L \tQ_\tau^\pi(\Ianchors,\Janchors) R)^{\dagger} \neq R^\dagger (\tQ_\tau^\pi(\Ianchors,\Janchors))^\dagger L^\dagger$ in general, and thus our estimation is different from $\tQ_\tau^\pi (s,\Janchors)(\tQ_\tau^\pi(\Ianchors,\Janchors))^\dagger \tQ_\tau^\pi(\Ianchors,a)$ used in \cite{shah2020sample}. However, weighting estimates by inverse leverage scores as proposed in Section \ref{subsec:leveraged_CUR} still provides unbiased estimates, in the following sense:
\begin{lemma}
Assume that $\vert \Ianchors \vert , \vert \Janchors \vert = \nachors$ and $\rank(Q^\pi) = d$.
Then:
\begin{align*}
    \forall (s,a)\in \states \times \actions:\quad Q^\pi(s,a) = Q^\pi(s,\Janchors)R (L Q^\pi(\Ianchors,\Janchors) R)^{\dagger} L Q^\pi(\Ianchors,a)
\end{align*}
\label{lemma:cross_completion_unbiased}
\end{lemma}
\paragraph{Proof of Theorem \ref{thm:cross_Qmax_guarantee}.} The proof follows from the proof of Proposition 13 of \cite{shah2020sample}, to which we refer the reader for a more detailed exposition. Based on Lemma \ref{lemma:cross_completion_unbiased} 
and following the proof of Proposition 13 in \cite{shah2020sample} (see (22) and (23) in \cite{shah2020sample}), we have $\forall (s,a) \in \states \times \actions$ and $(*) := \vert \hQ^\pi(s,a) - Q^\pi(s,a) \vert$:
\begin{align}
    (*) \leq
    \sqrt{2} \Vert (L \tQ_\tau^\pi(\Ianchors, &\Janchors) R)^{\dagger} \Vert_\op \Vert L (\tQ_\tau^\pi(\Ianchors,a) \tQ_\tau^\pi(s,\Janchors) -  Q^\pi(\Ianchors,a) Q^\pi(s,\Janchors)) R  \Vert_{\F}  \nonumber \\
    & +
    \Vert (L \tQ_\tau^\pi(\Ianchors,\Janchors) R)^{\dagger} - (L Q^\pi(\Ianchors,\Janchors) R)^{\dagger} \Vert_\op \Vert L Q^\pi(\Ianchors,a) Q^\pi(s,\Janchors)R\Vert_{\F}
    \label{eq:Qbar_decomposition_proof}
\end{align}

We will repeatedly use result from Lemma \ref{lemma:conc_DI_DJ} and condition on the event when given bounds on $\Vert L \Vert_\op$ and $\Vert R \Vert_\op$ hold. 
We begin by bounding the first term in \eqref{eq:Qbar_decomposition_proof}. Using the assumption that $\forall (s,a)\in \Ianchors\times \Janchors$: $\vert \tQ_\tau^\pi(s,a) - Q^\pi(s,a)\vert \leq \varepsilon_\square$ we obtain:
\begin{align*}
    \Vert L (\tQ_\tau^\pi(\Ianchors,\Janchors) - Q^\pi(\Ianchors,\Janchors)) R  \Vert_\op &\leq \Vert L \Vert_\op \nachors \Vert \tQ_\tau^\pi(\Ianchors,\Janchors) - Q^\pi(\Ianchors,\Janchors) \Vert_{\infty} \Vert R \Vert_\op \\ &\leq c_\Ianchors  c_\Janchors   \varepsilon_\square \sqrt{\cardS \cardA}  \log^2 \left( \frac{\cardS+\cardA}{\delta} \right).
\end{align*}
Combining this inequality with our assumption on $\varepsilon_\square$ and Corollary \ref{corr:sigma_d_lower_bound} with $\eta = 1/4$ gives:
\begin{align*}
    \Vert (L \tQ_\tau^\pi(\Ianchors,\Janchors) R)^{\dagger} \Vert_\op  &= \frac{1}{\sigma_d(L \tQ_\tau^\pi(\Ianchors,\Janchors) R)}  \\ &\leq \frac{1}{\sigma_d(L Q^\pi(\Ianchors,\Janchors) R) - \Vert L (\tQ_\tau^\pi(\Ianchors,\Janchors) - Q^\pi(\Ianchors,\Janchors)) R  \Vert_\op} \\ &
    \leq \frac{8}{\sigma_d(Q^\pi)}
\end{align*}

Second term in \eqref{eq:Qbar_decomposition_proof} can be bounded as follows:
\begin{align*}
    \Vert L (\tQ_\tau^\pi(\Ianchors,a) \tQ_\tau^\pi(s,\Janchors) & - Q^\pi(\Ianchors,a) Q^\pi(s,\Janchors))  R \Vert_{\F} \\ &\leq \Vert L \Vert_\op \Vert\tQ_\tau^\pi(\Ianchors,a) \tQ_\tau^\pi(s,\Janchors) - Q^\pi(\Ianchors,a) Q^\pi(s,\Janchors) \Vert_{\F} \Vert R \Vert_\op 
\end{align*}
and then use that:
\begin{align*}
    \Vert\tQ_\tau^\pi(\Ianchors,a) \tQ_\tau^\pi(s,\Janchors) - Q^\pi(\Ianchors,a) Q^\pi(s,\Janchors) \Vert_{\F} \leq \sqrt{\vert \Ianchors\vert \vert \Janchors \vert} (2 \varepsilon_+ \Vert Q^\pi\Vert_{\infty} + \varepsilon_+^2)
\end{align*}
Combining this result with Lemma \ref{lemma:conc_DI_DJ} we get:
\begin{align*}
    \Vert L (\tQ_\tau^\pi(\Ianchors,a) \tQ_\tau^\pi(s,\Janchors) - Q^\pi(\Ianchors,a) Q^\pi(s,\Janchors)) R \Vert_{\F} \leq c_\Ianchors  c_\Janchors  (2\Vert Q^\pi\Vert_{\infty}\varepsilon_+ + \varepsilon^2_+) \sqrt{\cardS \cardA} \log^2 \left( \frac{\cardS+\cardA}{\delta} \right)
\end{align*}

Similarly to the proof of Proposition 13 in \cite{shah2020sample}, we bound the third term from \ref{eq:Qbar_decomposition_proof} using inequality $\Vert B^\dagger - A^\dagger \Vert_\op \leq \frac{1+\sqrt{5}}{2}\min\{ \Vert A^\dagger \Vert_\op^2, \Vert B^\dagger\Vert_\op^2 \} \Vert B-A\Vert_\op$ as follows:
\begin{align*}
    \Vert (L \tQ_\tau^\pi(\Ianchors,\Janchors) R)^{\dagger} - (L Q^\pi(\Ianchors,\Janchors) R)^{\dagger} \Vert_\op \leq 64 c_\Ianchors  c_\Janchors  \frac{\varepsilon_\square }{\sigma_d^2(Q^\pi)} \sqrt{\cardS \cardA} \log^2 \left( \frac{\cardS+\cardA}{\delta} \right)
\end{align*}

And the last term from \eqref{eq:Qbar_decomposition_proof} can be bounded as follows:
\begin{align*}
    \Vert L Q^\pi(\Ianchors,a) Q^\pi(s,\Janchors)R\Vert_{\F} & \leq \Vert L \Vert_\op \Vert Q^\pi(\Ianchors,a) Q^\pi(s,\Janchors) \Vert_{\F} \Vert R \Vert_\op \\& \leq c_\Ianchors  c_\Janchors  \sqrt{\cardS \cardA} \Vert Q^\pi\Vert_{\infty}^2 \log^2 \left( \frac{\cardS+\cardA}{\delta} \right)
\end{align*}
where we used that $\Vert Q^\pi(\Ianchors,a) Q^\pi(s,\Janchors) \Vert_{\F} \leq \nachors \Vert Q^\pi \Vert_{\infty}^2$.

Combining all derived inequalities we obtain:
\begin{align*}
    \Vert \hQ^\pi - Q^\pi \Vert_{\infty} \leq 8 c_\Ianchors  c_\Janchors  &( 2\Vert Q^\pi\Vert_{\infty} \varepsilon_+ + \varepsilon_+^2) \frac{\sqrt{\cardS\cardA}}{\sigma_d(Q^\pi)} \log^2\left( \frac{\cardS+\cardA}{\delta} \right) \\ &+ 64 c_\Ianchors ^2 c_\Janchors ^2 \frac{\cardS\cardA}{\sigma_d(Q^\pi)^2} \varepsilon_\square \Vert Q^\pi \Vert_{\infty}^2 \log^4\left( \frac{\cardS+\cardA}{\delta} \right)
\end{align*}

\paragraph{Proof of Lemma \ref{lemma:cross_completion_unbiased}.} First, define matrices $\cD_U, \cD_W \in \mathbb{R}^{\nachors \times d}$ by $\cD_U = L U_{\Ianchors,:}$ and $\cD_W = R W_{\Janchors,:}$, and note that $\cD_U$ and $\cD_W$ are not orthogonal. However, we claim and prove in the end of this proof that: \begin{align}
    ( \cD_U \Sigma \cD_W^\top)^\dagger = (\cD_W^\top)^\dagger \Sigma^{-1} \cD_U^\dagger
    \label{eq:pinv_proof}
\end{align}
Moreover, since $\cD_U$ and $\cD_W$ have full column rank, we have that $\cD_U^\dagger \cD_U = I_{d\times d}$ and $\cD_W^\top (\cD_W^\top)^\dagger = I_{d\times d}$. Thus we have $\forall (s,a)\in \states \times \actions$:
\begin{align*}
    Q^\pi(s,\Janchors)R (L Q^\pi(\Ianchors,\Janchors) R)^{\dagger} L Q^\pi(\Ianchors,a) &= e_s^\top U\Sigma \cD_W^\top (\cD_U \Sigma \cD_W^\top)^\dagger \cD_U \Sigma W^\top e_a \\ &=   e_s^\top U\Sigma (\cD_W^\top (\cD_W^\top)^\dagger)  \Sigma^{-1}  (\cD_U^\dagger \cD_U) \Sigma W^\top e_a \\
    &= e_s^\top U\Sigma W^\top e_a = Q^\pi(s,a)
\end{align*}
Now we proceed with proving \eqref{eq:pinv_proof} following similar argument as in Lemma 1 in \cite{drineas2008relative}. Let SVD of $\cD_U$ and $\cD_W$ be given by $\cD_U = U_{\cD_U} \Sigma_{\cD_U}W_{\cD_U}^\top$ and $\cD_U = U_{\cD_W} \Sigma_{\cD_W}W_{\cD_W}^\top$. First, we use that $U_{\cD_U}$ and $U_{\cD_W}$ are orthogonal matrices to get:
\begin{align*}
    ( \cD_U \Sigma \cD_W^\top)^\dagger &= (U_{\cD_U} \Sigma_{\cD_U}W_{\cD_U}^\top \Sigma W_{\cD_W}\Sigma_{\cD_W}U_{\cD_W}^\top )^\dagger
    \\ &= U_{\cD_W} (\Sigma_{\cD_U}W_{\cD_U}^\top \Sigma W_{\cD_W}\Sigma_{\cD_W})^\dagger U_{\cD_U}^\top
\end{align*}
Since $\cD_U$ and $\cD_W$ are matrices with full column rank, all matrices inside of the pseudoinverse are of size $d\times d$ and full rank. Thus, their product is as well full rank and replacing pseudoinverse by inverse we obtain:
\begin{align*}
    ( \cD_U \Sigma \cD_W^\top)^\dagger = U_{\cD_W} \Sigma_{\cD_W}^{-1} W_{\cD_W}^\top \Sigma^{-1} W_{\cD_U} \Sigma_{\cD_U}^{-1} U_{\cD_U}^\top = (\cD_W^\top)^\dagger \Sigma^{-1} \cD_U^\dagger
\end{align*}

\subsection{Concentration results for the proof of Theorem \ref{thm:cross_Qmax_guarantee}}
\begin{lemma}
    Assume that $p$ is a probability measure on $\states$ such that $p_i \geq \eta \frac{\Vert U_{i,:}\Vert_2^2}{d}$ for all $i\in [\cardS]$ and some $\eta \in [0,1]$. Let $\Ianchors$ be a set obtained by sampling $\nachors$ entries of $\states$ according to $p$ i.e. for any $i\in [\cardS]: i\in \Ianchors$ with probability $\min \{ 1,\nachors p_i\}$. Define diagonal matrix $L$ with entries $\frac{1}{\min\{1,\sqrt{\nachors p_i}\} }$ for $i\in \Ianchors$ and matrix $\cD_U \in \mathbb{R}^{\nachors \times d}$ given by $\cD_U = L U_{\Ianchors,:}$. Then, for any $\delta \in (0,1)$:
    \begin{align*}
        \Vert \cD_U^\top \cD_U - I_{d\times d} \Vert_\op \leq 2\sqrt{\frac{d }{\nachors \eta} \log \left( \frac{2d}{\delta} \right)} 
    \end{align*}
     holds with probability at least $1-\delta$ whenever $\nachors \geq \frac{4d}{9\eta}\log(2d/\delta)$. 
    \label{lemma:conc_spectral_subsampled}
\end{lemma}
\begin{proof}
    First we argue that case $p_i > \frac{1}{K}$ is simple. Denote by $\states_+$ states for which $p_i \leq \frac{1}{K}$. Then, we have:
    \begin{align*}
        \Vert \cD_U^\top \cD_U - I_{d\times d} \Vert_\op &= \Vert U_{\Ianchors,:}^\top L^2 U_{\Ianchors,:} - U^\top U \Vert_\op \\ &= \left\Vert \sum_{i\in \Ianchors \cap \states_+} \delta_i (Z^{(i)})^\top Z^{(i)} L_{i,i}^2 -   \sum_{i\in \states_+ }  (Z^{(i)})^\top Z^{(i)} \right\Vert_\op,
    \end{align*}
    where $Z^{(i)}$ are obtained from $U$ by zeroing all rows except $i$-th, and $\delta_i$'s are i.i.d. Bernoulli($\nachors p_i$) for $i\in \states_+$. Now we can rewrite the first term:
    \begin{align*}
        \sum_{i\in \states_+} \delta_i^2 (Z^{(i)})^\top Z^{(i)} L_{i,i}^2 = \sum_{i\in \states_+} \frac{1}{\nachors p_i} \delta_i U_{i,:}^\top U_{i,:},  
    \end{align*}
    and take expectation over $\delta_i$'s to get $\EE \left[ \sum_{i\in \states_+} \delta_i^2 (Z^{(i)})^\top Z^{(i)} L_{i,i}^2 \right] = \sum_{i\in \states_+ }  (Z^{(i)})^\top Z^{(i)}$.

    Now, define $X^{(i)} = (\delta_i - \nachors p_i) \frac{1}{\nachors p_i} U_{i,:}^\top U_{i,:}$ for $i\in \states_+$. Note that:
    \begin{align*}
        \Vert X^{(i)}\Vert_\op \leq \frac{1}{\nachors p_i} \Vert U_{i,:} \Vert_2^2 \leq \frac{d}{\nachors \eta }
    \end{align*}
    by our assumption on $p$. Moreover, using that $\mathrm{Var}(\delta_i) = \nachors p_i(1-\nachors p_i) \leq \nachors p_i$ we have:
    \begin{align*}
        \EE  \left[ \sum_{i\in \states_+} X^{(i)}(X^{(i)})^\top \right] = \sum_{i\in \states_+} \EE (\delta_i - \nachors p_i)^2 \frac{1}{\nachors^2 p_i^2 } \Vert U_{i,:} \Vert_2^2 U_{i,:}^\top U_{i,:} \leq \frac{d}{\nachors \eta} \sum_{i\in \states_+}  U_{i,:}^\top U_{i,:}
    \end{align*}
    implying that $\Vert \EE  [ \sum_{i\in [\cardS]} X^{(i)}(X^{(i)})^\top ] \Vert_\op \leq \frac{d}{\nachors \eta}$. Finally, noting that $X^{(i)}$ are symmetric matrices $\forall i$, we apply matrix Bernstein inequality to obtain:
    \begin{align*}
        \PP(\Vert \cD_U^\top \cD_U - I_{d\times d} \Vert_\op \geq t) = \PP \left( \Big\Vert \sum_{i\in [\cardS]} X^{(i)} \Big\Vert_\op \geq t \right) \leq 2d \exp\left( - \frac{\nachors \eta}{2d}\frac{t^2}{1 + \frac{t}{3}} \right)
    \end{align*}
    Setting right hand side equal to $\delta$ finishes the proof.
\end{proof}

\begin{corollary}
If anchor states of size at least $\nachors \geq \frac{16d}{\eta}\log(4d/\delta)$ are chosen according to Lemma \ref{lemma:conc_spectral_subsampled}, we have with probability $\geq 1-\delta$:
\begin{align*}
    \sigma_d(L Q^\pi(\Ianchors,\Janchors) R) = \sigma_d(\cD_U \Sigma \cD_W^\top) \geq \sigma_d(\cD_U) \sigma_d(Q^\pi)  \sigma_d(\cD_W) \geq \frac{1}{4}\sigma_d(Q^\pi) 
\end{align*}
\label{corr:sigma_d_lower_bound}
\end{corollary}
Note that we could use inequality above since $\cD_U$ and $\cD_W$ have full column rank.

\begin{lemma}
Consider setting of Lemma \ref{lemma:conc_spectral_subsampled}. Then there exist universal constants $c_\Ianchors ,c_\Janchors >0$ such that with probability at least $1-\delta$:
    \begin{align*}
        \Vert L \Vert_\op \leq c_\Ianchors   \sqrt{\frac{\cardS}{\nachors}}\log\left( \frac{\cardS}{\delta} \right) , \qquad \Vert R \Vert_\op \leq c_\Janchors  \sqrt{\frac{\cardA}{\nachors}}\log\left( \frac{\cardA}{\delta} \right)
    \end{align*}
    \label{lemma:conc_DI_DJ}
\end{lemma}
\begin{proof}
    We note that if $L_{i,i} = 1$ (i.e. $\nachors p_i \geq 1$), then obviously the inequality above holds for $S\geq \nachors$, and thus we consider only cases where $L_{i,i} = \frac{1}{\sqrt{\nachors p_i}}$. Now, note that $\Vert L \Vert_\op = \Vert L^{\mathrm{ext}}\Vert_\op$, where $L^{\mathrm{ext}} = \sum_{i=1}^{\cardS} \delta_i \frac{1}{\sqrt{\nachors p_i}} e_i e_i^\top$, and where $\delta_i$ are i.i.d. Bernoulli($\nachors p_i)$. Next, we have: $\EE [\delta_i \frac{1}{\sqrt{\nachors p_i}} e_i e_i^\top] = \sqrt{\nachors p_i} e_i e_i^\top$, and thus:
    \begin{align*}
        \Vert \EE L^{\mathrm{ext}} \Vert_\op = \sqrt{\nachors \max_i p_i}  \leq \sqrt{\nachors} 
    \end{align*}
    Define $Y^{(i)} = (\delta_i - \nachors p_i) \frac{1}{\sqrt{\nachors p_i}}e_i e_i^\top$. We have $\EE [ Y^{(i)}(Y^{(i)})^\top ] = (1-\nachors p_i) e_i e_i^\top$, and hence the variance term in matrix Bernstein is upper bounded by $1$. Lastly by our assumption on $p$ we have for all $i\in[\cardS]$:
    \begin{align*}
        \Vert Y^{(i)} \Vert_\op \leq \frac{1}{\sqrt{\nachors p_i}} \leq c \sqrt{\frac{\cardS}{\nachors}}
    \end{align*}
    By matrix Bernstein we obtain:
   \begin{align*}
        \PP(\Vert L^{\mathrm{ext}} - \EE L^{\mathrm{ext}}\Vert_\op \geq t) = \PP \left( \Big\Vert \sum_{i\in [\cardS]} Y^{(i)} \Big\Vert_\op \geq t \right) \leq 2\cardS \exp\left( -\frac{\frac{t^2}{2}}{1 + t\frac{c}{3}\sqrt{\frac{\cardS}{\nachors}}} \right)
    \end{align*}
    Equating last term with $\delta$ and using that $S,A \gg d$ we obtain statement of the lemma.
\end{proof}

\newpage
\section{Sample Complexity Analysis of $\lorapi$} \label{app:lora}

In this appendix, we present the proof of Theorem \ref{thm:lora-pi-lme}. It  is a direct consequence of the performance guarantee of $\lme$ (see Theorem \ref{thm:lme-guarantee}) and an error bound on approximate policy iteration, which we provide in this appendix (see Lemma \ref{lem:API-convergence}).

\subsection{Proof of Theorem \ref{thm:lora-pi-lme}}

\begin{proof}[Proof of Theorem \ref{thm:lora-pi-lme}] To start with, we first observe that, 
     according to Lemma \ref{lem:API-convergence}, $\lorapi$ outputs $\hat{\pi}$ with $\Vert V^\star - V^{\hat{\pi}}\Vert_\infty \le \varepsilon$, if it holds that 
\begin{align*}
    (i) & \qquad \gamma^{t-1} \Vert V^\star -  V^{\pi^{(1)}}  \Vert_\infty \le \frac{2r_{\max}\gamma^{(N_{\textup{epochs}})} }{1-\gamma} \le \frac{\varepsilon}{2} \\
    (ii) & \qquad \Vert \widehat{Q}^{(t)} - Q^{(t)}\Vert_\infty \le \frac{(1-\gamma)^2\varepsilon}{4}, \quad \forall t \in [N_{\textup{epochs}}]
\end{align*}
where we introduce the notation $Q^{(t)} := Q^{\pi^{(t)}}$ as a shorthand. Now, we note that condition $\emph{(i)}$ is satisfied if 
\begin{align*}
    N_{\textup{epochs}} = \left\lceil \frac{1}{1- \gamma}\log\left( \frac{4 r_{\max}}{(1-\gamma)\varepsilon }\right) \right\rceil
\end{align*}
which is already as chosen in $\lorapi$. Now, in order for $(ii)$ to hold we use Theorem \ref{thm:lme-guarantee}. We define the events:
\begin{align*}
    \cE_{t} = \left\lbrace \Vert \widehat{Q}^{(t)} - Q^{(t)}\Vert_\infty \le \frac{(1-\gamma)^2\varepsilon}{4} \right\rbrace 
\end{align*}
We show that $\cap_{t \in [N_{\textup{epochs}}]} \cE_t$ holds with high probability. To that end, it is sufficient to analyse for each $t \in [N_{\textup{epochs}}]$, the event $ 
 \cE_t^c$ conditionally on the event that $(\cap_{k\in [t-1]} \cE_k)$ holds. Indeed, by  
using the elementary inequality $\PP(\cE^c \cup \cB^c ) \le \PP( \cE^c \vert \cB ) + \PP(\cB^c)$ in a recursive manner, we can write  
\begin{align}
     \PP( (\cap_{t \in [N_{\textup{epochs}}]}  \cE_t )^c )  = \PP( \cup_{t \in [N_{\textup{epochs}}]}  \cE_t^c ) 
    & \le \sum_{t \in [N_{\textup{epochs}}]} \PP( \cE_{t}^c \vert \cap_{k \in [t - 1]}  \cE_k ) 
\end{align}
We will show that for all $t \in [N_{\textup{epochs}}]$, $\PP( \cE_{t}^c \vert \cap_{k \in [t - 1]}  \cE_k ) \le \delta/N_{\textup{epochs}}$, which would entail that $\PP( (\cap_{t \in [N_{\textup{epochs}}]}  \cE_t )^c ) \le \delta$ and ensure that $\Vert V^\star - V^{\hat{\pi}}\Vert_\infty \le 1-\varepsilon$ holds with probability at least $1 -\delta$.

Let  $t \in [N_{\textup{epochs}}]$. Note that by using Theorem \ref{thm:lme-guarantee}, we can immediately show that $\PP( \cE_{t}^c \vert \cap_{k \in [t - 1]}  \cE_k ) \le \delta/N_{\textup{epochs}}$ provided that 
\begin{align*}
    \frac{T}{N_{\textup{epochs}}} = \widetilde{\Omega}\left( \frac{ r_{\max}^2  \kappa^4 \alpha^2  d^2 \left( (S+A) + \alpha^2  d  \right) }{ (1-\gamma)^7 \varepsilon^2} \ \log^{10}\left( \frac{N_{\textup{epochs}}}{\delta}\right)   \right)
\end{align*}
which entails, by definition of $N_{\textup{epochs}}$ as chosen in $\lorapi$, an equivalent sample complexity to  
\begin{align*}
    T & = \widetilde{\Omega}\left( \frac{ r_{\max}^2  \kappa^4 \alpha^2  d^2 \left((S+A) + \alpha^2  d  \right) }{ (1-\gamma)^8 \varepsilon^2} \ \log^{10}\left( \frac{1}{(1-\gamma)\delta} \log\left(  \frac{r_{\max}}{(1-\gamma)\varepsilon}\right)  \right)  \log\left(  \frac{r_{\max}}{(1-\gamma)\varepsilon}\right) \right) \\
    & =  \widetilde{\Omega}\left( \frac{ r_{\max}^2  \kappa^4 \alpha^2  d^2 \left((S+A) + \alpha^2  d  \right) \log^{10}(e/\delta) \log(e/\varepsilon)}{ (1-\gamma)^8 \varepsilon^2}  \right)
\end{align*}
where we emphasize that $\widetilde{\Omega}(\cdot)$ may hide poly-log dependencies on $S$, $A$, $(1-\gamma)^{-1}$, $d$, $\kappa$, $\alpha$, $r_{\max}$, $\log(e/\varepsilon)$, $\log(e/\delta)$. This the desired sample complexity in Theorem \ref{thm:lora-pi-lme}.

Note that Theorem \ref{thm:lme-guarantee} also requires that 
\begin{align}\label{eq:eps-needed}
    \frac{(1-\gamma)^2 \varepsilon} {4} \leq \Vert Q^{(t)} \Vert_\infty
\end{align}
We show that this is satisfied by the condition $\varepsilon \lesssim \underline{\varepsilon}$. First, we show that under this condition, we have $ \Vert Q^{(1)}\Vert \le 2 \Vert Q^{(t)} \Vert_\infty$. Using Lemma \ref{lem:improvement}, and conditionally on the event $\cap_{k \in [t]}\cE_{k}$ holding, we have that for all $k < t$,   
\begin{align*}
    \cT^\star(V^{(k)}) \le V^{(k+1)} + \frac{2\epsilon}{1-\gamma} \mathbf{1} \le \cT^\star\left( V^{(k+1)} + \frac{2\varepsilon^{(k)}}{1-\gamma} \mathbf{1} \right) \le  \cT^\star( V^{(k+1)}) + \frac{2 \gamma \epsilon}{1-\gamma} 
\end{align*}
where $\epsilon = \frac{(1-\gamma)^2\varepsilon}{4}$, implying, in particular, that
\begin{align*}
    \Vert Q^{(k)} \Vert_\infty \le \Vert Q^{(k+1)} \Vert_\infty + \frac{2\gamma (1-\gamma) \varepsilon}{2}. 
\end{align*}
Summing the above inequalities from $1$ to $t-1$, together with the fact that $t-1 \le N_{\textup{epochs}}$, gives  
\begin{align*}
     \Vert Q^{(1)} \Vert_\infty \le \Vert Q^{(t)} \Vert_\infty + \frac{\gamma(1-\gamma) (t-1) \varepsilon }{2} \le \Vert Q^{(t)} \Vert_\infty + \frac{\gamma \varepsilon}{2} \log\left( \frac{4\Rmax}{(1-\gamma)^2 \varepsilon}\right).
\end{align*}
In view of this inequality, we note that $2 \Vert Q^{(t)} \Vert_\infty \ge \Vert Q^{(1)} \Vert_{\infty}$, if 
\begin{align*}
\gamma \varepsilon  \log\left( \frac{4\Rmax}{(1-\gamma)^2 \varepsilon}\right) \le  \Vert Q^{(1)} \Vert_\infty 
\end{align*}
We can verify that the above condition is implied by: 
\begin{align}\label{eq:eps-needed-2}
    \frac{1}{\varepsilon} \ge \frac{4 \gamma}{\Vert Q^{(1)}\Vert_\infty} \log\left( \frac{16 \gamma \Rmax}{(1-\gamma)^2 \Vert Q^{(1)}\Vert_\infty}\right) \iff \varepsilon \le \frac{\Vert Q^{(1)} \Vert_\infty}{2 \gamma  \log\left( \frac{16 \gamma \Rmax}{(1-\gamma)^2 \Vert Q^{(1)}\Vert_\infty}\right) }
\end{align}
where we used the elementary fact $ x \ge 2a\log(2a) + 2 b \implies x \ge a\log(x) + b$ for all $a, b>0$.

Thus, from \eqref{eq:eps-needed-2} we conclude that the condition on $\varepsilon$, \eqref{eq:eps-needed}, is satisfied if the following condition holds:
\begin{align*}
     \varepsilon  \leq \min\left( 1 ,  \frac{1} {2 \gamma  \log\left( \frac{16 \gamma \Rmax}{(1-\gamma)^2 \Vert Q^{(1)}\Vert_\infty}\right) } \right) \Vert Q^{(1)} \Vert_\infty.
\end{align*}
This is the desired condition on $\varepsilon$ in Theorem \ref{thm:lora-pi-lme}. With this we have concluded the proof.    
\end{proof}

\subsection{Error Bound for Approximate Policy Iteration}

The following result, a standard variant of Proposition 6.2 in  \cite{bertsekas1996neuro}, shows that the described approximate policy iteration is guaranteed to converge within an $\epsilon$-accuracy.  
\begin{lemma}\label{lem:API-convergence} Let $(\pi^{(t)})_{t \ge 1}$ be a sequence of deterministic policies selected recursively as described in $\lorapi$, and denote $V^{(t)} =  V^{\pi^{(t)} }$ for all $t\ge 1$. Let $\epsilon > 0$ and suppose that for all $t \ge 1$, it holds that
\begin{align*}
      \Vert \widehat{Q}^{(t)} - Q^{(t)} \Vert_\infty \le \epsilon.
\end{align*}
Then, for all $t \ge 1$, we have 
    \begin{align*}
            \Vert V^\star - V^{(t+1)} \Vert \le \gamma^{t} \Vert V^\star -  V^{(1)}  \Vert_\infty  + \frac{2\epsilon}{(1- \gamma)^2}. 
    \end{align*}
\end{lemma}
The proof of Lemma \ref{lem:API-convergence} follows standard arguments, but we provide it for completeness.

\begin{lemma}\label{lem:improvement} Let $\pi$ be a deterministic policy, and assume that  
$
\Vert \widehat{Q}^\pi - Q^\pi \Vert_\infty \le \epsilon 
$.
Assume that policy  $\pi'$ is selected greedily with respect to $\widehat{Q}^\pi$, i.e., for all $s \in \cS$, $\pi'(s) = \argmax_{a \in \cA} \widehat{Q}^\pi(s, a)$, then 
\begin{align*}
         V^\pi \le \cT^\star(V^\pi) \le V^{\pi'} + \frac{2\epsilon}{1-\gamma} \mathbf{1}.
\end{align*}  
\end{lemma}

\begin{proof}[Proof of Lemma \ref{lem:improvement}] Before we proceed with the proof, let us define the composition of a deterministic policy $\pi''$ and a $Q^\pi$ function, $\pi'' \circ Q^\pi (s) := Q^\pi(s, \pi''(s))$. We know that 
    \begin{align*}
        V^\pi & = \pi \circ Q^\pi   \le \max_{\pi''} \pi'' \circ Q^\pi = \cT^\star (V^\pi)
    \end{align*}
    where $\le$ is applied component-wise.
    Next, we have  
    \begin{align*}
        V^\pi \le \cT^\star(V^\pi)  = \max_{\pi''} \pi'' \circ Q^\pi  & \le \max_{\pi''} \pi'' \circ \widehat{Q}^\pi + \max_{\pi''} \pi'' \circ (Q^\pi - \widehat{Q}^\pi) \\
        & \le \pi' \circ \widehat{Q}^\pi + \epsilon \mathbf{1} \\
        & \le \pi' \circ Q^\pi + \pi' \circ(  \widehat{Q}^\pi- Q^\pi) + \epsilon \mathbf{1} \\
        & \le \pi' \circ Q^\pi + 2 \epsilon \mathbf{1} \\
        & \le \cT_{\pi'}(V^\pi) + 2 \varepsilon \mathbf{1}
    \end{align*}
    where $\cT_{\pi}$ is the Bellman policy evaluation operator. By monotonicity of the operator $\cT_{\pi}$, we can re-iterate 
    \begin{align*}
        \cT_{\pi'}(V^\pi) & \le \cT_{\pi'} (   \cT_{\pi'} (V^\pi) + 2 \epsilon \mathbf{1}) \le \cT_{\pi'}^2 (V^\pi) + 2 \gamma \epsilon \mathbf{1}.
    \end{align*}
    Thus, we finally obtain 
    \begin{align*}
        V^\pi \le \cT^\star(V^\pi) \le  \cT_{\pi'}^{k+1}(V^\pi) + 2\epsilon \left(\sum_{t = 0}^{k}  \gamma^t \right) \mathbf{1}.
    \end{align*}
    Taking $k\to \infty$, we get 
    \begin{align*}
         V^\pi \le \cT^\star(V^\pi) \le \cT_{\pi'}(V^\pi) + 2\epsilon \mathbf{1} \le  V^{\pi'} + \frac{2\epsilon}{1-\gamma} \mathbf{1}.
    \end{align*}
\end{proof}

\begin{proof}[Proof of Lemma \ref{lem:API-convergence}] 
    We start by noting that, thanks to Lemma \ref{lem:improvement}, we have: for all $t\ge 1$,
    \begin{align*}
      V^{(t+1)} + \frac{2\epsilon}{1- \gamma}\mathbf{1} \ge \cT^\star(V^{(t)}), 
    \end{align*}
    where $\ge$ is applied component-wise.
    Thus, applying this inequality recursively we obtain
    \begin{align*}
         V^{(t+1)} + \frac{2\epsilon}{1 -\gamma} \mathbf{1}  & \ge \cT^\star \left(V^{(t)} + \frac{2\epsilon}{1- \gamma} \mathbf{1}\right) - \frac{2\epsilon \gamma }{1- \gamma} \mathbf{1} \\
         & \ge (\cT^\star)^2\left(V^{(t-1)}\right) -  \frac{2\epsilon \gamma }{1- \gamma} \mathbf{1} \\ 
         & \ge (\cT^\star)^{t} (V^{(1)}) - \frac{2\epsilon}{1- \gamma} \left(\sum_{k = 1}^{t -1} \gamma^{k} \right) \mathbf{1} \\
         & \ge (\cT^{\star})^{t} (V^{(1)}) - \frac{2\epsilon(1- \gamma^{t})}{(1- \gamma)^2} \mathbf{1} +  \frac{2\epsilon}{1- \gamma} \mathbf{1},
    \end{align*}
    which gives at the end
    \begin{align*}
        V^{(t+1)}  & \ge  (\cT^\star)^{(t)} (V^{(1)}) - \frac{2\epsilon(1- \gamma^{t})}{(1- \gamma)^2} \mathbf{1}
    \end{align*}
    Thus, we have 
\begin{align*}
    V^\star - V^{(t+1)} & \le V^\star - (\cT^\star)^{t} (V^{(1)})  + \frac{2\epsilon(1- \gamma^{t})}{(1- \gamma)^2} \mathbf{1} 
    \le (\cT^\star)^{t}(V^\star) - (\cT^\star)^{t} (V^{(1)})  + \frac{2\epsilon(1- \gamma^{t})}{(1- \gamma)^2} \mathbf{1}  
\end{align*}
Thus, using the contraction property of $\cT^\star$, and that $1-\gamma^t \le 1$, we have 
\begin{align*}
    \Vert V^\star - V^{(t+1)} \Vert_\infty & \le \Vert(\cT^\star)^{t}(V^\star) - (\cT^\star)^{t} (V^{(1)})  \Vert_\infty  + \frac{2\epsilon}{(1- \gamma)^2} 
      \le \gamma^{t} \Vert V^\star -  V^{(1)}  \Vert_\infty  + \frac{2\epsilon}{(1- \gamma)^2}.
\end{align*}  
\end{proof}

\newpage
\section{Extension of Guarantees to Approximately-Low Rank MDPs}
\label{subsec:approximately_low_rank}
We consider the setting where the matrix $Q^\pi$ is approximately low rank. Specifically, we define a constant $\zeta_d$ such that $\zeta_d = \Vert Q^\pi(s,a) - Q^\pi_d(s,a) \Vert_{\infty}$, where $Q_d^\pi$ is the best $d$-rank approximation of $Q^\pi$ in the operator norm. Note that $\zeta_d \leq \sigma_{d+1}(Q^\pi) \leq \sqrt{SA}\zeta_d$. In contrast to Theorem \ref{thm:row-wise-guarantee}, where the additional perturbation term $\Delta$ arises from a controllable quantity (through roll-out length $\tau$), here we assume that $\zeta_d$ is fixed in advance and unknown. For simplicity, we omit terms stemming from the $\Delta$ perturbation, but the results still hold in that setting. Here, we show that if:
\begin{align}
    \zeta_d = \widetilde{O} \left( \sigma_d(Q^\pi)  \min\left\{  \frac{\sqrt{d}}{S+A}, \frac{1}{\kappa \sqrt{SA}}  \right\} \right)\tag{$A_+$}
    \label{assumpt_approx_low}
\end{align}
we can obtain similar guarantees for $\Vert V^\star - V^{\hat{\pi}}\Vert_\infty$ as in Theorem \ref{thm:lora-pi-lme} even in the approximate low rank setting, with an additive error scaling with $\widetilde{O}(\frac{1}{1-\gamma} \zeta_d d\kappa^2 \alpha^2)$. Next, we show that our three main theorems still hold in this setting.

\paragraph{Theorem \ref{thm:leverage_scores}: Leverage scores estimation.}
We can repeat the arguments from the proof of Theorem \ref{thm:row-wise-guarantee} to obtain, with high probability, $\forall s \in [\cardS]$:
\begin{align*}
    \Vert U_{s,:} - \hU_{s,:} O_{\widehat{U}} \Vert_2 & = \widetilde{O} \left( \balpha\left( \sqrt{\frac{d}{\Ttau}} + \kappa \Vert U_{s,:} \Vert_2 \sqrt{\frac{\cardS+\cardA}{\Ttau}} \right)  +\frac{ \textcolor{blue}{\zeta_d} \sqrt{\cardS+\cardA }}{\sigma_{d}(Q^\pi)}   + \kappa\Vert U_{s,:} \Vert_2 \frac{\textcolor{blue}{\sigma_{d+1}(Q^\pi)} }{\sigma_d(Q^\pi)} \right)
\end{align*}
if $\Ttau = \widetilde{\Omega}\left( \balpha^2(\cardS+\cardA)\right)$, and $\sigma_{d+1}(Q^\pi) \leq \sigma_{d}(Q^\pi)/64$. New terms are highlighted in blue in the inequality above. A similar inequality holds for the rows of the matrix of right singular vectors $W$.

Under Assumption \ref{assumpt_approx_low} and using that $\sigma_{d+1}(Q^\pi) \leq \sqrt{SA}\zeta_d$, we have:
\begin{align*}
    \frac{\zeta_d\sqrt{\cardS+\cardA}}{\sigma_d(Q^\pi)} = \widetilde{O} \left( \frac{\sqrt{d}}{\sqrt{\cardS+\cardA }} \right),\qquad \mathrm{and} \qquad \kappa \Vert U_{s,:}\Vert_2 \frac{\sigma_{d+1}(Q^\pi)}{\sigma_d(Q^\pi)} = \widetilde{O} \left( \Vert U_{s,:} \Vert_2 \right)
\end{align*}
indicating that the contributions of the two newly added terms are negligible for leverage score estimation and that Theorem \ref{thm:leverage_scores} still holds in this setting.

\paragraph{Theorem \ref{thm:lme-guarantee}: Complete matrix estimation.} Theorem \ref{thm:cross_Qmax_guarantee} holds with the same arguments. Instead of Lemma \ref{lemma:Hoeffding_cross_entries}, we have that with high probability: $\forall (s,a) \in (\Ianchors \times \actions) \cup (\states \times \Janchors)$:
    \begin{align*}
        \vert \tQ_\tau^\pi (s,a) - Q^\pi(s,a) \vert \leq \frac{\Rmax}{1-\gamma} \sqrt{\frac{2}{N}\log\left( \frac{4 \nachors (\cardS + \cardA)}{\delta} \right)} + \zeta_d 
    \end{align*}
Note that our conditions on $\zeta_d$ and $\sigma_{d+1}(Q^\pi)$ ensure that the conditions on $\varepsilon_\square$ and $\varepsilon_+$ in Theorem \ref{thm:cross_Qmax_guarantee} ($\varepsilon_\square \lesssim \frac{\sigma_d(Q^\pi)}{\sqrt{\cardS \cardA}} \log^{-2}\left( \frac{\cardS + \cardA}{\delta} \right)$, $\varepsilon_+ \lesssim \Vert Q^\pi\Vert_{\infty}$) still hold, as:
\begin{align*}
    \zeta_d  = \widetilde{O} \left( \frac{\sqrt{d}\sigma_d(Q^\pi)}{S+A} \right) = \widetilde{O} \left( \frac{\Vert Q^\pi \Vert_\F}{ S+A} \right) = \widetilde{O} \left( \frac{\sqrt{SA}\Vert Q^\pi \Vert_\infty}{S+A} \right) = \widetilde{O} \left( \Vert Q^\pi \Vert_\infty \right)
\end{align*}
Then, the upper bound on $ \Vert \hQ^\pi - Q^\pi \Vert_{\infty} $ from Theorem \ref{thm:cross_Qmax_guarantee} will include an additive term: $
 \zeta_d  \frac{\cardS\cardA\Vert Q^\pi \Vert_{\infty}^2}{\sigma_d^2(Q^\pi)}\log^4 \left( \frac{S+A}{\delta} \right) = \widetilde{O} \left( \zeta_d d\kappa^2 \alpha^2 \right)$. Finally, under approximate low-rank structure, Theorem \ref{thm:lme-guarantee} guarantees that with high probability, if $\varepsilon \lesssim \Vert Q^\pi \Vert_{\infty}$ and $
        T = \widetilde{\Omega}_\delta\left( \frac{ (S+A) + \alpha^2  d   }{ (1-\gamma)^3 \varepsilon^2} (r_{\max}^2  \kappa^4 \alpha^2 d^2  )  \right)$, we have $\Vert \widehat{Q}^\pi - Q^\pi \Vert_\infty \leq \varepsilon + \widetilde{O}\left(\zeta_d d\kappa^2 \alpha^2\right)$.

This aligns with Theorem 14 in \cite{sam2023overcoming}, where the approximation error scales by terms corresponding to $\frac{\cardS\cardA\Vert Q^\pi \Vert_{\infty}^2}{\sigma_d^2(Q^\pi)}$ in our setting, as both methods use CUR-like matrix recovery.

\paragraph{Theorem \ref{thm:lora-pi-lme}: Guarantee for \lorapi.} Based on the approximate policy iteration theorem, which claims:
\begin{align*}
    (1-\gamma)\Vert V^\star - V^{\hat{\pi}}  \Vert_\infty \le 2 r_ {\max}\gamma^{N_{\textup{epochs}}} + 2 \max_{t \in [N_{\textup{epochs}}]} \Vert \widehat{Q}^{(t)} - Q^{\pi^{(t)}}\Vert_\infty.
\end{align*}
we observe that the error from approximate low rank propagates through the second term, yielding an additive error of magnitude $\frac{1}{1-\gamma} \zeta_d d\kappa^2 \alpha^2$ to the error of Theorem \ref{thm:lora-pi-lme}.

\newpage
\section{Miscellaneous Results}

In this section, we provide some of the observations and results about the truncated value matrix. More specifically, we present the  proof to Lemma \ref{lem:truncated-Q}, and a discussion on the variance proxy of the truncated discounted sum of rewards. 

\subsection{Truncated Value Matrix}
\label{subsec:app_truncated_value_lemma}

\begin{proof}[Proof of Lemma \ref{lem:truncated-Q}]
We know that $Q^\pi$ satisfies the following identity: for all $(s,a) \in \cS \times \cA$, 
\begin{align*}
    Q^\pi(s,a) & = \EE\left[ \sum_{t=0}^\infty \gamma^t r_t^\pi \bigg \vert (s_0^\pi, a_0^\pi) = (s,a) \right] \\
    & = Q^\pi_\tau(s,a)   + \gamma^\tau \EE\left[ \sum_{t=\tau+1}^\infty \gamma^{t-\tau} r_t^\pi \bigg \vert (s_0^\pi, a_0^\pi) = (s,a) \right]. 
\end{align*}
Furthermore, note that 
\begin{align*}
   \left \vert  \EE\left[ \sum_{t=\tau+1}^\infty \gamma^{t-\tau} r_t^\pi \bigg \vert (s_0^\pi, a_0^\pi) = (s,a) \right] \right \vert \le \Rmax \sum_{t=0}^\infty \gamma^t = \frac{\Rmax}{1- \gamma}.
\end{align*}
and thus 
\begin{align*}
\Vert Q^\pi_\tau - Q^\pi \Vert_\infty \le \frac{\gamma^\tau \Rmax}{1-\gamma} \leq \frac{\Rmax}{1-\gamma}\exp(-\tau(1-\gamma)) .
\end{align*}
where we used that $\gamma \leq \exp(\gamma-1)$ for $\gamma \in (0,1)$. Setting the right hand side of the last inequality equal to $\epsilon$ we obtain statement of the lemma.

\end{proof}

\subsection{Equivalent Noise Model}
\label{app:noise_equivalence}
Recall definition of $\widetilde{Q}^\pi_\tau$ from \eqref{eq:emp-truncated-q}: 
\begin{align*}
     \widetilde{Q}^\pi_\tau(s,a) & = \frac{SA}{N} \sum_{k=1}^{N}\left(\sum_{t=0}^\tau \gamma^t r_{k, t}^\pi \right)\indicator_{\lbrace (s_{k,0}^\pi, a_{k,0}^\pi) = (s,a)\rbrace}, \quad \forall (s,a)\in \cS \times \cA.
\end{align*}
Consider one of $N$ sampled trajectories with index $k$ starting from $(s_{k,0}^\pi, a_{k,0}^\pi)=(s,a)$, and note that
\begin{align*}
\left\vert \sum_{t=0}^\tau \gamma^t r_{k, t}^\pi \right\vert \le \frac{\Rmax}{1-\gamma}.
\end{align*}
Moreover, since $Q^\pi_\tau$ is given by:
\begin{align*}
   Q^\pi_\tau(s,a) & = \EE^\pi\left[ \sum_{t=0}^\infty \gamma^t r_t \indicator_{\lbrace t \le \tau\rbrace}\big \vert s_0^\pi = s, a_0^\pi = a\right],
\end{align*}
we have  
\begin{align*}
\EE\left[\sum_{t=0}^\tau \gamma^t r_{k, t}^\pi \indicator_{\lbrace (s_{k,0}^\pi, a_{k,0}^\pi) = (s,a)\rbrace}  \bigg \vert (s_{k,0}^\pi, a_{k,0}^\pi)  = (s,a) \right] = Q^\pi_\tau (s,a) \indicator_{\lbrace (s_{k,0}^\pi, a_{k,0}^\pi) = (s,a)\rbrace} 
\end{align*}
In other words, each term inside of the outer loop in definition of $\tQ_\tau^\pi$ is uniformly bounded and equal to 
$Q^\pi_\tau (s,a) \indicator_{\lbrace (s_{k,0}^\pi, a_{k,0}^\pi) = (s,a)\rbrace} $ in expectation. Thus we can view estimate $\tQ^\pi_\tau(s,a)$ equivalently as:
\begin{align*}
\widetilde{Q}^\pi_\tau(s,a) = \frac{SA}{N} \sum_{k=1}^{N} 
 (Q^\pi_\tau(s,a) + \xi_{s,a,k}) \indicator_{\lbrace (s_{k,0}^\pi, a_{k,0}^\pi) = (s,a)\rbrace} 
\end{align*}
where $\xi_{s,a,k}$ are i.i.d. across $k$, and $\vert Q^\pi_\tau(s,a) + \xi_{s,a,k} \vert  \le \frac{\Rmax}{1-\gamma}$, implying that $\xi_{s,a,k}$ are $\frac{2\Rmax}{1-\gamma}$-subgaussian random variables.

Next note that the number of times $N(s,a) = \sum_{k=1}^{N} \indicator{\lbrace (s_{k,0}^\pi, a_{k,0}^\pi) = (s,a)\rbrace}$ that we sample entries are random variables with multinomial distribution, since $\PP((s_{k,0}^\pi, a_{k,0}^\pi) = (s,a)) = \frac{1}{\cardS \cardA}$ and $\sum_{(s,a)}N(s,a)=N$. This weak dependence between the entries can be dealt with using the Poisson approximation argument (see Section C.2 in \cite{stojanovic2024spectral}). Essentially, this enables us to rewrite matrix $\tQ^\pi_\tau$ as a matrix with i.i.d. entries. Namely, we have for all $(s,a)$:
\begin{align*}
    \widetilde{Q}^\pi_\tau(s,a) = \frac{SA}{N} \sum_{k=1}^{Y(s,a)} 
    (Q^\pi_\tau(s,a) + \xi_{s,a,k}) 
 \end{align*}
where $Y(s,a)$ are i.i.d. Poisson random variables with parameter $\EE [Y(s,a)] = \frac{N}{SA}$. The fact that the two noise models are equivalent is depicted in Lemma 20 in \cite{stojanovic2024spectral} claiming that probability of an event under the multinomial model can be upper bounded by $\sqrt{T}$ times probability of the same event under the Poisson model. Practically, this adds a multiplicative factor of $T$ in our probabilistic claims. For more thorough exposition of this issue check Section C.2 in \cite{stojanovic2024spectral}.



\end{document}